\documentclass[twoside]{article} 
\usepackage[final]{colm2026_conference}
\usepackage[T1]{fontenc}
\usepackage{microtype}
\usepackage{hyperref}
\usepackage{url}
\usepackage{booktabs}
\usepackage{textcomp}
\usepackage{color-edits}
\usepackage{titletoc}
\addauthor[Deqing]{df}{magenta}
\addauthor[Tianyi]{tz}{teal}
\addauthor[Misha]{mb}{blue}
\addauthor[Vatsal]{vs}{red}
\addauthor[Robin]{rj}{blue}

\usepackage{lineno}

\definecolor{darkblue}{rgb}{0, 0, 0.5}
\definecolor{modblue}{HTML}{DCEBFA}
\hypersetup{colorlinks=true, citecolor=darkblue, linkcolor=darkblue, urlcolor=darkblue}

\usepackage{graphicx}

\title{Convergent Evolution: How Different Language Models\\Learn Similar Number Representations}

\newcommand{\usc}{\ensuremath{\textcolor[HTML]{990000}{\sigma}}}
\newcommand{\ucsd}{\ensuremath{\textcolor[HTML]{182B49}{\psi}}}
\author{\textbf{Deqing Fu}$^{\usc}$\hspace{2mm}
\textbf{Tianyi Zhou}$^{\usc}$\hspace{2mm}
\textbf{Mikhail Belkin}$^{\ucsd}$\hspace{2mm}
\textbf{Vatsal Sharan}$^{\usc}$\hspace{2mm}
\textbf{Robin Jia}$^{\usc}$\hspace{2mm}\\
$^{\usc}$University of Southern California \hspace{2mm} $^{\ucsd}$UC San Diego\\
\texttt{\{deqingfu,tzhou029,vsharan,robinjia\}@usc.edu, mbelkin@ucsd.edu}
}

\usepackage[table,dvipsnames]{xcolor}

\usepackage{amsmath,amsfonts,bm}

\def\eqref#1{equation~\ref{#1}}

\def\1{\bm{1}}

\def\vzero{{\bm{0}}}

\def\vmu{{\bm{\mu}}}

\def\ve{{\bm{e}}}

\def\vv{{\bm{v}}}

\def\vF{{\bm{F}}}

\def\mE{{\bm{E}}}

\def\mM{{\bm{M}}}

\def\mS{{\bm{S}}}

\def\mU{{\bm{U}}}

\DeclareMathAlphabet{\mathsfit}{\encodingdefault}{\sfdefault}{m}{sl}
\SetMathAlphabet{\mathsfit}{bold}{\encodingdefault}{\sfdefault}{bx}{n}

\def\sC{{\mathbb{C}}}

\def\sR{{\mathbb{R}}}

\DeclareMathOperator{\Tr}{Tr}

\usepackage{amsthm}
\usepackage{cleveref}
\usepackage{thmtools}
\usepackage{thm-restate}
\declaretheorem[name=Theorem]{theorem}
\declaretheorem[name=Lemma, sibling=theorem]{lemma}

\declaretheorem[name=Remark, sibling=theorem]{remark}

\usepackage{makecell}
\usepackage{tabularx}
\usepackage{pifont}
\newcommand{\cmark}{\ding{51}}

\usepackage{hwemoji}
\newcommand{\huggingface}{\scalerel*{\includegraphics[page=2380]{hwemoji-assets.pdf}}{X}}

\ifcolmpreprint
\fancypagestyle{firstpage}{
  \fancyhf{}
  \lhead{\vspace*{-\headsep}\vspace*{-\topmargin}\vspace*{-1in}\includegraphics[height=5mm]{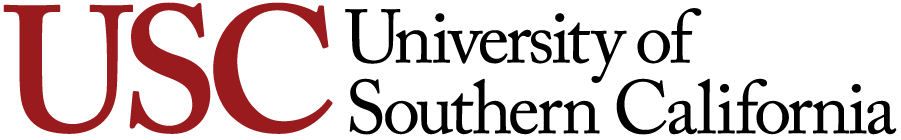}\hspace{1em}\includegraphics[height=5mm]{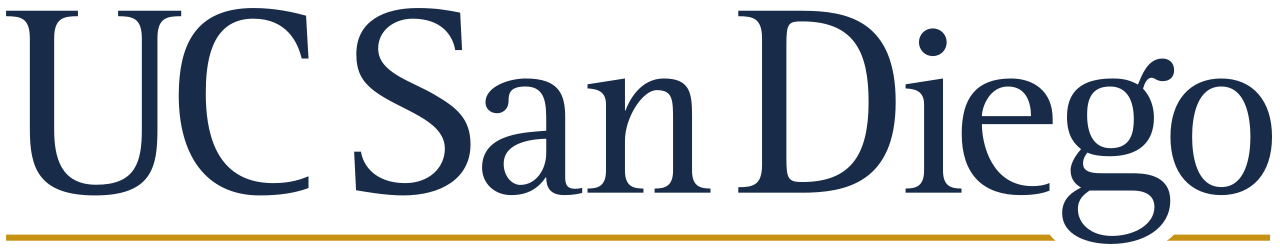}}
  \rhead{April 22, 2026}
}

\fancypagestyle{otherpages}{
  \fancyhf{}
  \fancyhead[RO]{\textit{Fu, Zhou, Belkin, Sharan, \& Jia}}
  \fancyhead[LE]{ \textit{Convergent Evolution}}
  \fancyhead[RE,LO]{\thepage}
}

\fi

\usepackage{titlesec}
\titlespacing\paragraph{0pt}{0pt}{2pt plus 1pt minus 1pt}
\titlespacing\section{4pt}{4pt}{0pt}
\titlespacing\subsection{2pt}{2pt}{0pt}
\begin{document}

\ifcolmsubmission
\linenumbers
\fi

\maketitle

\begin{abstract}
Language models trained on natural text learn to represent numbers using periodic features with dominant periods at $T=2, 5, 10$.
In this paper, we identify a two-tiered hierarchy of these features: while Transformers, Linear RNNs, LSTMs, and classical word embeddings trained in different ways all learn features that have period-$T$ spikes in the Fourier domain, only some learn \emph{geometrically separable} features that can be used to linearly classify a number mod-$T$.
To explain this incongruity, we prove that Fourier domain sparsity is necessary but not sufficient for mod-$T$ geometric separability.
Empirically, we investigate when model training yields geometrically separable features, finding that the data, architecture, optimizer, and tokenizer all play key roles.
In particular, we identify two different routes through which models can acquire geometrically separable features: they can learn them from complementary co-occurrence signals in general language data, including text-number co-occurrence and cross-number interaction, or from multi-token (but not single-token) addition problems.
Overall, our results highlight the phenomenon of \emph{convergent evolution} in feature learning: A diverse range of models learn similar features from different training signals.

{\Large\huggingface} Models: \href{https://hf.co/collections/deqing/convergent-evolution}{https://hf.co/collections/deqing/convergent-evolution}
\end{abstract}

\section{Introduction} \label{sec:introduction}



Language models trained on natural language develop periodic representations for number tokens.
For many Transformer-based language models, \citet{zhou2024pre} show that the embeddings  of integer tokens have consistent spikes in the Fourier domain at periods $T = 2, 5$, and $10$.
Such periodic features have also been well-documented in models'
intermediate representations \citep{levy2025language} and in the model mechanisms that implement addition \citep{zhou2024pre, kantamneni2025language}.
\citet{engels2025not} even find analogous periodic structures for other cyclical concepts such as days of the week and months of the year.
These findings have been broadly interpreted as evidence that language models learn structured numerical representations through next-token prediction.

In this paper, we first demonstrate that this phenomenon is far more general than previously recognized.
\Cref{fig:fourier_universality} shows that the same $T = 2, 5, 10$ spikes appear not only in Transformers \citep{transformer} of varying scale (GPT-2 \citep{gpt2}, GPT-OSS \citep{gptoss}, Llama-3 \citep{llama3}, Llama-4 \citep{llama4}, and DeepSeek-V3 \citep{deepseekv3}), but also in non-Transformer LLMs (Mamba \citep{mamba}, Falcon-Mamba \citep{falcon-mamba}, xLSTM \citep{xlstm}, Kimi-Linear \citep{kimi-linear}) and classical word embeddings (GloVe \citep{pennington-etal-2014-glove} and FastText \citep{bojanowski-etal-2017-fasttext}).
Even the raw token frequency distribution of numbers in the training corpus, with no model at all, exhibits the same periodic spectrum (see \Cref{fig:fourier_no_probe}).
We view this universality as a case of \textit{convergent evolution}: different systems independently develop the same representation because they share the same constraints from training data and tokenization.
In biology, convergent evolution refers to the independent emergence of similar traits in unrelated organisms facing shared environmental pressures, such as the independent evolution of eyes in vertebrates and cephalopods \citep{McGhee2011convergent}.
Fourier features in number embeddings are analogous: a shared trait that arises from shared constraints on the learning process.

But do these Fourier spikes indicate that models have learned
functional numerical structure? We identify a two-tiered hierarchy of periodic features: only some systems with Fourier spikes cleanly encode modular arithmetic properties in their embeddings. 
By this we mean that the residue class $n \bmod T$ is linearly decodable from the embedding $\ve(n)$. Period-$T$ features naturally group numbers by their value mod $T$, and the linear representation
hypothesis \citep{park2023lrh} conjectures that such structure should be accessible via linear probes.
We call the emergence of Fourier spikes \textit{spectral convergence} and the emergence of linearly separable mod-$T$ classes \textit{geometric convergence}.
Spectral convergence appears in almost every system we examine, but
geometric convergence does not: Transformers and linear RNNs trained on 10 billion tokens develop embeddings where mod-$T$ classes are linearly separable, while LSTMs trained on identical data develop more prominent Fourier spikes but achieve chance-level probing.
Understanding what separates these two levels of convergence is the central question of this paper. We summarize our contributions below.

\begin{figure}[t]
    \centering
    \includegraphics[width=0.49\linewidth]{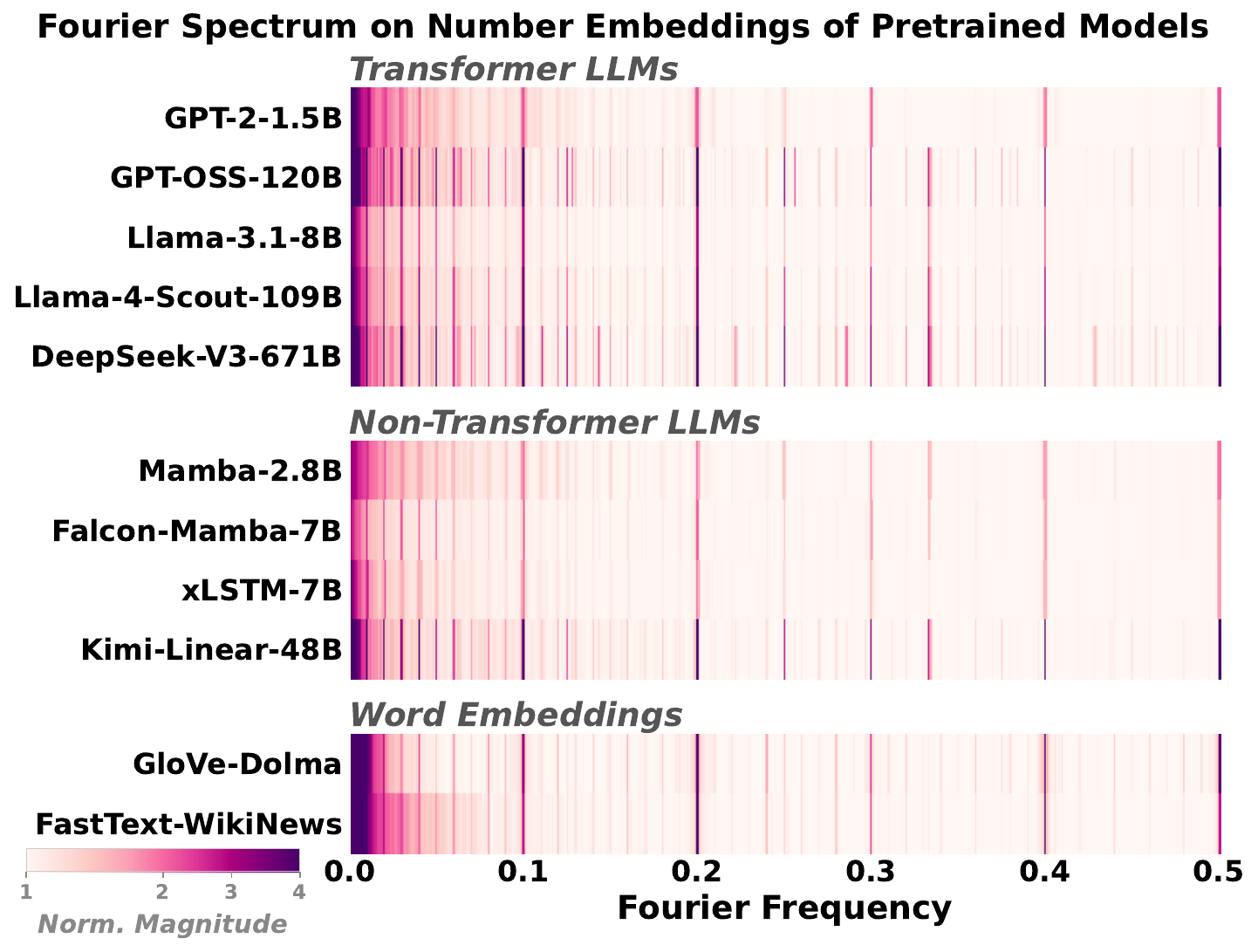}
    \includegraphics[width=0.49\linewidth]{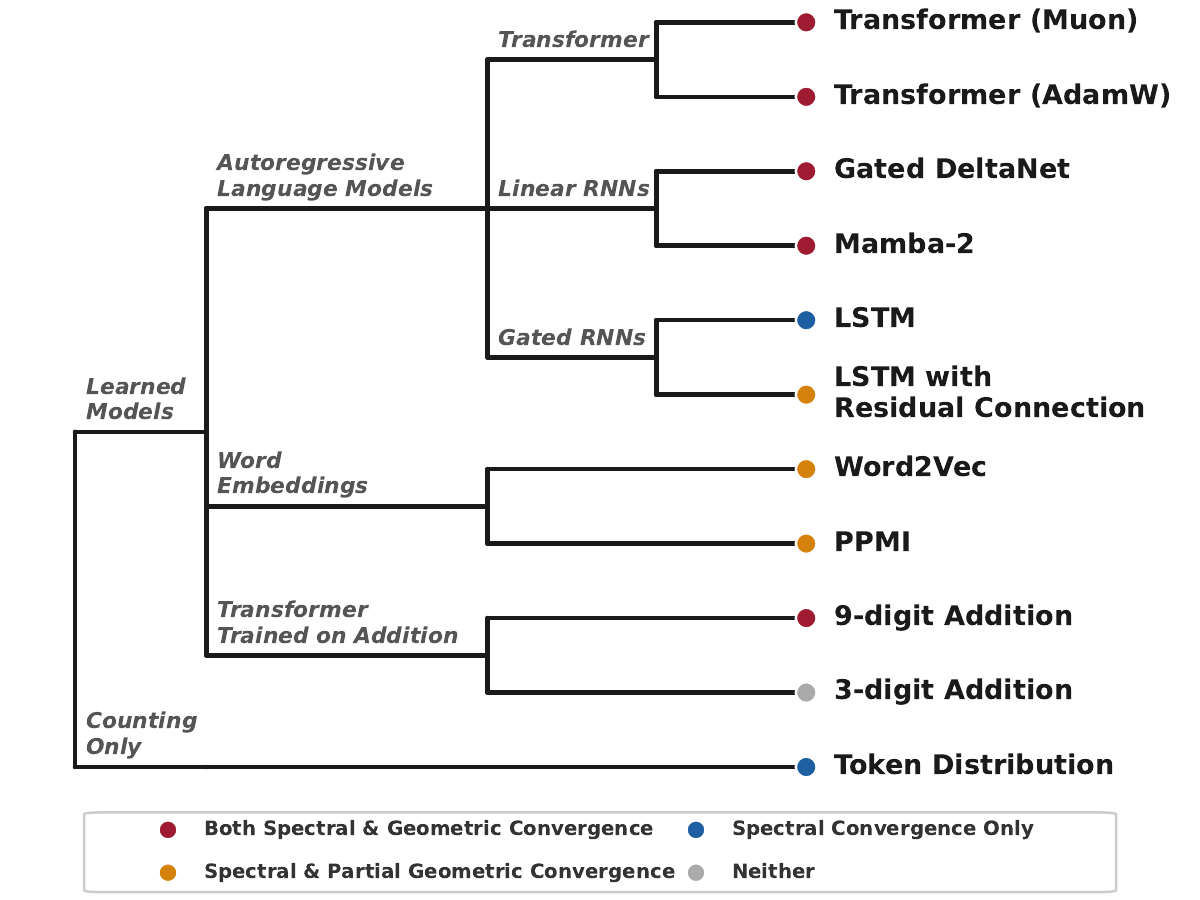}
    \caption{\textbf{Universality of Fourier Features and Convergent Evolution}. (\textit{Left}) Fourier spectrum of number embeddings across three architecture families: Transformer LLMs, non-Transformer LLMs, and classical word embeddings. Each row shows the median-normalized magnitude at each Fourier frequency. All models exhibit consistent spikes at frequencies of period $T = 2, 5, 10$, etc. (\textit{Right}) Convergent Evolution of various models studied in this paper. It shows two types of convergence: \textit{spectral convergence} where models learn Fourier spikes, and \textit{geometric convergence} where models learn modular probes.}
    \label{fig:fourier_universality}
\end{figure}

\paragraph{Fourier spikes are universal but probing is not.} We show that Fourier spikes at $T = 2, 5, 10$ appear in every system we examine, spanning Transformer and non-Transformer LLMs, classical word embeddings, and even the raw token frequency distribution (Figure~\ref{fig:fourier_universality}). We demonstrate both theoretically (\Cref{thm:fourier}) and empirically (\Cref{fig:fourier_no_probe}) that Fourier spikes are necessary but not sufficient for mod-$T$ probing, and explain how models with similar spectra can have vastly different probing accuracy (\S\ref{sec:problem_setup}).

\paragraph{Geometric convergence requires data, architecture, and optimizer to align.}
Through controlled experiments on 300M-parameter models trained on identical data, we isolate three factors that jointly determine whether mod-$T$ classes become linearly separable in the embeddings.
Our experimental methodology can be viewed as a form of \emph{structure attribution}. Analogous to how influence functions \citep{koh2017influence} or Shapley values \citep{ghorbani2019data} attribute model predictions to individual training examples, our controlled perturbations attribute the emergence of learned representations to specific structural properties of the data distribution.
We find that geometric convergence depends on several complementary data signals: perturbations that progressively remove text-number co-occurrence,
cross-number interaction, or context length each degrade probing,
while Fourier spikes persist across all conditions
(\S\ref{ssec:data}). 
The architecture plays a critical role: Transformers and linear RNNs achieve strong probing while LSTMs trained on the same data remain at chance (\S\ref{ssec:arch_opt}) unless we add residual connections (\S\ref{app:lstm}). Regardless of the optimizer, both Transformers and linear RNNs learn the same Fourier spectrum but different probing performance.

\paragraph{Convergent evolution takes a different form under arithmetic task pressure.}
In models trained on addition from scratch, the tokenizer determines whether Fourier structure emerges (\S\ref{sec:addition}).
Multi-token addition requires computing each output token as a sum modulo 1000, forcing the model to solve modular subproblems that produce circular representations.
Single-token addition admits multiple strategies, and the resulting representations randomly depend on the optimizer.

\section{Related Work} \label{sec:related_work}


\paragraph{Fourier Features.}
Fourier features have been used for many years in computer vision as edge and orientation detectors \citep{olshausen1997sparse, olah2020overview, fiquet2023polar}. The original Transformer applies sinusoidal position encodings \citep{transformer}, and several works have found that explicitly injecting high-frequency components into inputs helps with spatial and numerical tasks \citep{tancik2020fourier, he2023frequency, hua2024fourier}.
Recently, this structure is also found to emerge without being designed in. Transformers trained on modular addition embed numbers on a circle and rotating to compute the answer \citep{nanda2023progress, zhong2023clock, gromov2023grokking}. The same holds in pretrained LLMs that number token embeddings break into Fourier components, and there are recognizable addition circuits in the attention and MLP layers \citep{zhou2024pre, kantamneni2025language, levy2025language}.  \citet{zhou2025fone} show that hard-coding these Fourier features improves arithmetic learning. 
These studies document spectral structure but do not test whether it implies geometric separability — a distinction our work shows is critical.
In this paper, we further study the question these papers leave open, where the structure comes from in the first place.

\paragraph{Mechanistic Interpretability.}
Mechanistic interpretability aims to reverse-engineer the  representations and algorithms learned by language models. 
The linear representation hypothesis \citep{park2023lrh} conjectures that high-level concepts, if learned, should be linearly decodable from model representations. Probing \citep{orgad2025llms, kossen2024semantic} is the standard tool for this.
However, number representations are not linearly encoded \citep{nanda2023progress, zhong2023clock, gromov2023grokking}. \citet{karkada2026symmetry} show that the winding geometry LLMs learn for months, years, and locations also arises in word2vec, and trace it theoretically to translation symmetry in the co-occurrence statistics. We attribute it empirically, showing that Fourier structure is universal but linearly separable geometry emerges only under specific data, architecture, tokenization, and optimizer choices. 
\citet{allen2025physics} use controlled synthetic pretraining to isolate which capabilities emerge from which architectural and data choices. 
Separately, \citet{huh2024platonic} argue that representations across models and modalities are converging toward a shared statistical model of reality, measured via global kernel alignment. We ask a different question: whether models that converge on the similar representation of a specific concept have learned the same functional structure, and show that they can diverge fundamentally.
In this paper, we vary tokenization, architecture, optimizer, and task to understand the convergent evolution of number representations.

\section{Problem Setup and Preliminary Analysis}
\label{sec:problem_setup}

\begin{figure}[t]
    \centering
    \includegraphics[width=\linewidth]{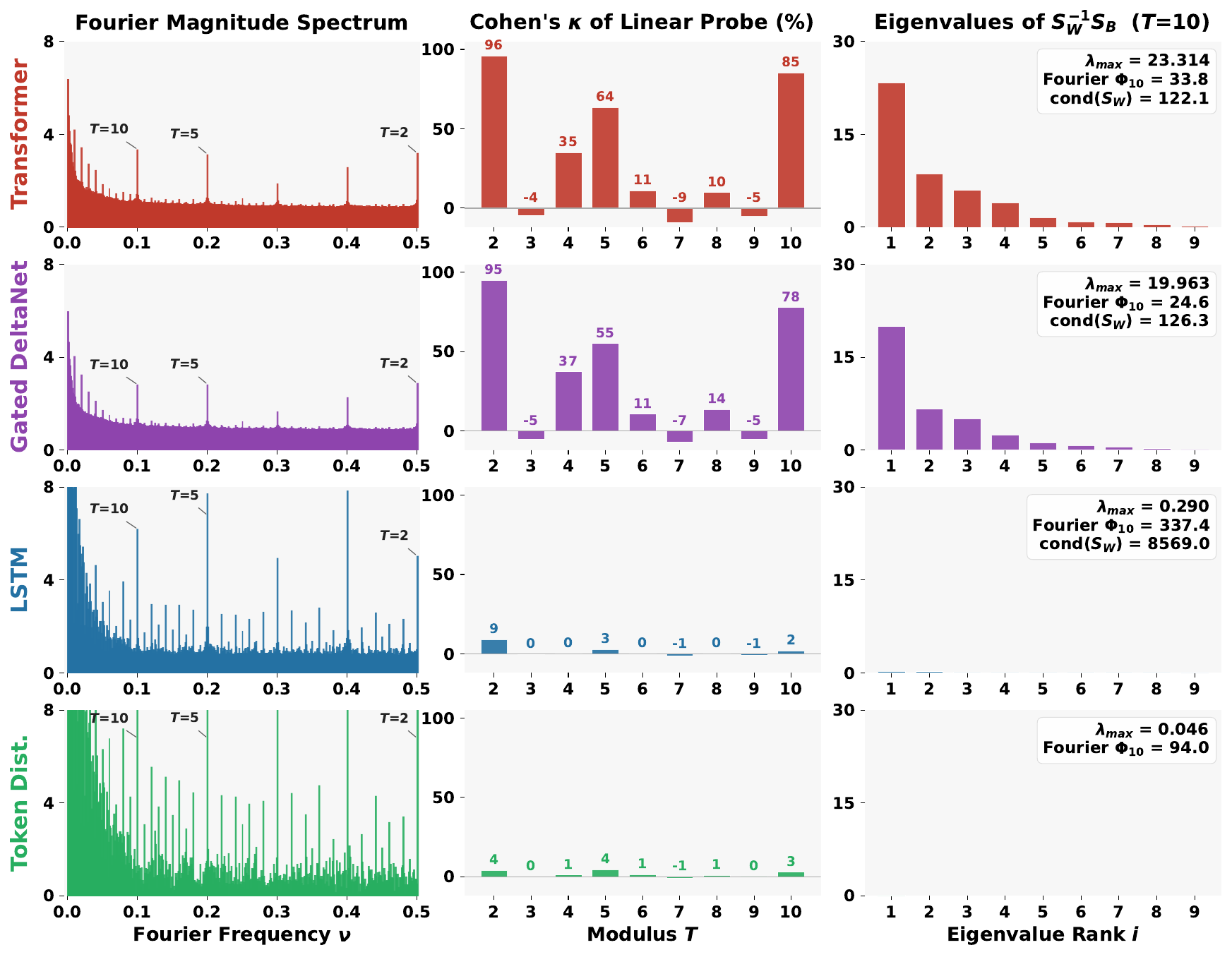}
    \caption{\textbf{A ``Spiky'' Fourier Spectrum Does Not Imply Good Feature Learning.} (\textit{Left}) Token embeddings of {\color{Maroon}Transformer}, {\color{Purple}Gated DeltaNet} and {\color{RoyalBlue}LSTM}, and even simply the {\color{Green}Number Token Distribution Frequency} exhibit distinct Fourier spikes at $T = 2, 5$, and $10$. (Middle) Linear probes reveal only the Transformer and Gated DeltaNet learned functional modular arithmetic with high Cohen's $\kappa$, while others remain at chance. (\textit{Right}) \Cref{thm:fourier} explains this discrepancy through the internal noise structure of the $T=10$ embeddings.}
    \label{fig:fourier_no_probe}
\end{figure}

We study the token embeddings of numbers $0$ through $N-1$ in language models, where $N = 1000$ corresponds to the set of numbers that receive single-token representations in the Llama-3 tokenizer \citep{llama3}.
Let $\ve(n) \in \sR^d$ denote the token embedding of number $n$.
To detect periodic structure in these embeddings, following \citet{zhou2024pre}, we compute the discrete Fourier transform along the token index:
\[
\vF_\nu = \frac{1}{\sqrt{N}} \sum_{n=0}^{N-1} \ve(n) \, e^{-2\pi i \nu n} \in \sC^d, \qquad \nu = \frac{k}{N},\; k = 0, \ldots, N-1.
\]
The power at frequency $\nu$ is $\|\vF_\nu\|^2 = \sum_{j=1}^d |F_\nu^{(j)}|^2$.
A \emph{Fourier spike at period $T$} refers to a visible peak in $\|\vF_{1/T}\|^2$ relative to neighboring frequencies. For a period $T$ dividing $N$, define the mod-$T$ residue classes $C_r = \{n : n \equiv r \pmod{T}\}$ for $r = 0, \ldots, T-1$, each of size $|C_r| = N/T$.
To evaluate whether the embeddings encode modular arithmetic at period $T$, we train a \emph{linear probe} ($T$-class logistic regression) to predict $n \bmod T$ from $\ve(n)$.

A natural question is whether the presence of a Fourier spike at period $T$ guarantees good $\bmod$-$T$ probing accuracy. As we will show in \Cref{fig:fourier_no_probe}, the answer is strikingly no: an LSTM trained on the same data as a Transformer develops larger Fourier power at $T = 10$ yet achieves chance-level probing. 
The following result shows how this is possible and demonstrates that the presence of a Fourier spike is a \textit{necessary} but not a \textit{sufficient} condition for learning modular probes.

\begin{restatable}{theorem}{maintheorem}
\label{thm:fourier}
Given embeddings $\{\ve(n)\}_{n=0}^{N-1}$ and residue classes $\{C_r\}_{r=0}^{T-1}$ defined above, let the class means $\vmu_r$, grand mean $\vmu$, between-class scatter matrix $\mS_B$, and within-class scatter matrix $\mS_W$ be
\[
\vmu_r = \frac{1}{|C_r|}\sum_{n \in C_r} \ve(n), \qquad \vmu = \frac{1}{N}\sum_{n=0}^{N-1} \ve(n),
\]
\[
\mS_B = \frac{1}{T}\sum_{r=0}^{T-1}(\vmu_r - \vmu)(\vmu_r - \vmu)^\top, \qquad \mS_W = \frac{1}{N}\sum_{r=0}^{T-1}\sum_{n \in C_r}(\ve(n) - \vmu_r)(\ve(n) - \vmu_r)^\top.
\]
Let $\Phi_T = \sum_{\ell=1}^{T-1} \|\vF_{\ell/T}\|^2$ be the total power at harmonics of period $T$, and $H_T = \{0, 1/T, 2/T, \ldots, (T{-}1)/T\}$ be the set of harmonic frequencies.
\begin{enumerate}
\item[\emph{(i)}]
If $\Phi_T = 0$, then $\mS_B = \bm{0}$ and no linear probe can classify $n \bmod T$ above chance.

\item[\emph{(ii)}] 
For any $T \ge 2$, $C > 0$, and $\varepsilon > 0$, there exist $N$ divisible by $T$ and embeddings $\ve(n)$ satisfying $\Phi_T > C$ yet no $T$-class linear classifier achieves accuracy above $1/T + \varepsilon$, i.e., no more than $\varepsilon$ above random guessing.

\end{enumerate}
\end{restatable}


\Cref{thm:fourier} establishes that $\Phi_T > 0$ is necessary but not sufficient. 
A natural follow-up is: what \emph{quantitatively} determines whether a spiky $\Phi_T$ translates into geometric separability? 

\begin{remark}
The gap between two parts of \Cref{thm:fourier} can be made more precise through the lens of Fisher's Linear Discriminant Analysis \citep{fisher1936lda}. Assume $d \geq T - 1$ and $S_W$ is invertible.
While a linear probe for $T$ classes separates data across a $(T-1)$-dimensional subspace, the absolute ceiling on its performance is dictated by its principal axis of separation: the single direction $v$ that maximizes the ratio of between-class to within-class variance, given by the Rayleigh quotient $\frac{v^\top \mS_B v}{v^\top \mS_W v}$.
The maximum achievable separability along this optimal axis is exactly the largest generalized eigenvalue of the scatter matrices, $\lambda_{\max}(\mS_W^{-1} \mS_B)$. We show (see \S~\ref{app:bounds}) that this maximum discriminant satisfies:
\[
\frac{1}{(T-1) \cdot \mathrm{cond}(\mS_W) } \cdot \frac{\Phi_T}{N \cdot \lambda_{\min}(\mS_W)} \;\leq\; \lambda_{\max}(\mS_W^{-1} \mS_B) \;\leq\; \frac{\Phi_T}{N \cdot \lambda_{\min}(\mS_W)}.
\]
The Fourier power $\Phi_T$ only guarantees a large total variance among the class means, which drives $\Phi_T = \Tr(\mS_B) / N$ in the numerator. It ensures the class centers are dispersed in space. However, actual linear separability depends on how this signal aligns with the within-class scatter, $\mS_W$. A large condition number, $\text{cond}(\mS_W) = \lambda_{\max}(\mS_W)/\lambda_{\min}(\mS_W)$, drastically lowers the minimum bound. If the periodic signal aligns with the dimensions of maximum within-class noise, the massive $\Phi_T$ is entirely swallowed by within-class variance. This drives $\lambda_{\max}$ toward zero and forces the entire $(T-1)$-dimensional subspace to collapse into overlapping classes.
\end{remark}

\begin{figure}[t]
    \centering
    \includegraphics[width=\linewidth]{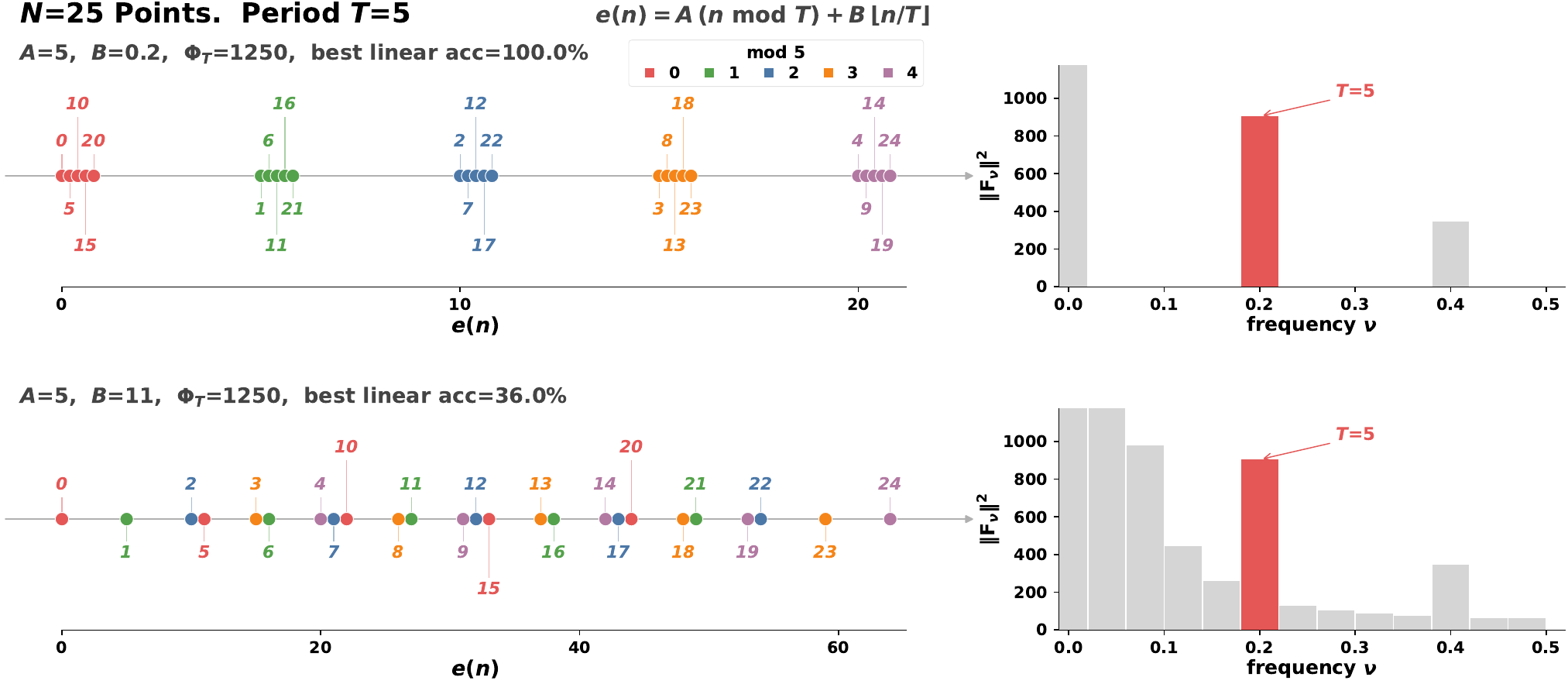}
    \caption{\textbf{Examples for the Proof of \Cref{thm:fourier} Part (ii) in \S~\ref{app:insufficient}}. In the proof, every number $n \in \{0,\dots,N-1\}$ has a unique decomposition $n = r + mT$ with residue $r = n \bmod T \in \{0,\dots,T{-}1\}$ and block index $m = \lfloor n/T \rfloor \in \{0,\dots,K{-}1\}$, where $K = N/T$.
    We set the embedding {\color{Maroon}$e(n) = Ar + Bm$}, so that $A$ controls the between-class scatter $S_B$ (and hence $\Phi_T = N \cdot S_B$) while $B$ controls only the within-class scatter $S_W$, with neither parameter affecting the other.
    As illustrated in the figure, fixing $A$ produces a persistent Fourier spike at period~$T$ regardless of~$B$: when $B$ is small the residue classes cluster into linearly separable groups, but as $B$ grows the embeddings interleave across classes so that the best linear classification accuracy goes closer to random guessing $1/T$.}
    \label{fig:thm_part_ii_examples}
    \vspace{-2ex}
\end{figure}

The proof of \Cref{thm:fourier} is given in Appendix~\ref{app:proof-fourier}, where part (i) follows naturally from Fourier identities in \Cref{lem:fourier-variance-identity} which shows that $\Tr(\mS_B) = \Phi_T / N$, and part (ii) constructs examples (see \Cref{fig:thm_part_ii_examples}) so that $\Phi_T$ can be arbitrarily large but residue classes are not linearly separable.
We now ask two empirical questions: how common are Fourier spikes in practice, and do they actually predict mod-$T$ probe accuracy?


Throughout this paper, unless stated otherwise, we train models with around 300 million parameters on 10B tokens from FineWeb-Edu \citep{lozhkov2024fineweb-edu} using the Llama-3 tokenizer, which assigns single tokens to each integer from 0 to 999 such that $N = 1000$. We include architecture details in the Appendix \ref{app:model_training}.
We evaluate number representations using two complementary measurements: (1) the \emph{Fourier magnitude spectrum}, computed as the power $\|\vF_\nu\|^2$ normalized by the median across frequencies $\nu$, and (2) the \emph{mod-$T$ probe accuracy}, measured by Cohen's $\kappa$ of a classifier trained to predict $n \bmod T$ from the token embedding $\ve(n)$. Cohen's $\kappa$ adjusts for the baseline accuracy of random guessing ($1/T$ for balanced classes), so that $\kappa = 0$ corresponds to chance and $\kappa = 100\%$ to perfect classification regardless of $T$. We use three probe types: linear (logistic regression), MLP, and RFM kernel \citep{radhakrishnan2024rfm}; unless noted, we report linear probe results, with the others in \S~\ref{app:mlp_rfm}. We report probe performance averaged over 30 runs: 3 random seeds, each with 10-fold cross-validation. The first measurement detects periodic structure while the second tests whether that structure supports mod-$T$ classification.

\paragraph{Fourier spikes are universal.} \Cref{fig:fourier_universality} shows that every  pretrained model we examine, spanning Transformer LLMs, non-Transformer LLMs,  and classical word embeddings, exhibits peaks at $\nu = 1/10, 1/5, 1/2$. \textit{Spectral convergence} is universal as long as natural languages are used for training. This holds across architectures across fundamentally different learning algorithms and across models never explicitly trained on numerical tasks. 

\paragraph{Fourier spikes do not imply modular arithmetic.}
Figure~\ref{fig:fourier_no_probe} presents a controlled comparison under our training setup: a Transformer, a Gated DeltaNet \citep{gdn}, an LSTM \citep{lstm} (all around 300M parameters), and the raw number token frequency distribution from the training corpus, where each number $n$ is represented by its scalar corpus frequency $p_n$ via counting rather than a learned embedding.
The left column shows that all four produce qualitatively similar Fourier spikes at periods $T = 2, 5, 10$. The middle column shows mod-$T$ probe accuracy: the Transformer and Gated DeltaNet achieve $\kappa = 96$ and $95$ at $T = 2$, and $\kappa = 85$ and $78$ at $T = 10$, while the LSTM and the token distribution remain at chance across all moduli. The right column explains this gap through Theorem~\ref{thm:fourier}.
The LSTM has \emph{larger} Fourier power $\Phi_{10}$ than the Transformer, yet its Fisher discriminant $\lambda_{\max}(\mS_W^{-1}\mS_B)$ is two orders of magnitude smaller.
The difference lies in the condition number $\mathrm{cond}(\mS_W)$: the LSTM's within-class scatter is highly anisotropic, so the periodic signal is buried under within-class variance and the classes overlap despite visible Fourier spikes.
Several recent studies have inspected Fourier spectra and concluded that models have learned modular structure \citep{nanda2023progress,zhou2024pre}.
Our analysis shows why this inference is unreliable: a visible spike at period $T$ guarantees $\Tr(\mS_B) > 0$, but probe accuracy depends on the eigenspectrum of $\mS_W^{-1}\mS_B$, which the Fourier power spectrum alone does not determine.

We call this arrangement of embeddings into linearly separable mod-$T$ classes \emph{Geometric Convergence}. 
Unlike spectral convergence, geometric convergence is selective: only certain combinations of data, architecture, and optimizer produce it.
When those align, however, it is robust across scale: models from 124M to
671B parameters achieve near-perfect mod-2, 5, and 10 probes, even for
multi-token numbers (\S\ref{app:mod_probe_more},
\S\ref{app:multi_token_probe}).
Both are instances of \textbf{convergent evolution}: different systems arriving at similar representations because of shared constraints. 
The next sections identify what drives each.

\section{Convergent Evolution in Language Model Pretraining}
\label{sec:pretraining}

Spectral convergence requires only the periodic frequency distribution of number tokens, but geometric convergence is selective.
We now ask which constraints must be present for geometric convergence to emerge. 
Through controlled experiments that vary one factor at a time, we identify three: the data signal the model receives (\S\ref{ssec:data}), the architecture, and the optimizer (\S\ref{ssec:arch_opt}).
All experiments use 300M-parameter models trained on 10B tokens from FineWeb-Edu \citep{lozhkov2024fineweb-edu} with the Llama-3 tokenizer, unless stated otherwise.

\subsection{Structural Attribution to Data} \label{ssec:data}
To isolate the environmental pressures driving convergence, we fix the architecture (300M Transformer) and optimizer (Muon, \citealp{jordan2024muon}) and vary only the training data.
We apply controlled perturbations (\Cref{tab:data_configs}) that each remove a specific type of co-occurrence while leaving others intact, attributing the emergence of spectral and geometric convergence to specific structural properties of the data.


\paragraph{Spectral convergence requires only token frequencies.}
All perturbations produce nearly identical Fourier spectra (\Cref{fig:data}, left), including \texttt{Unigram Replace}, which destroys all co-occurrence structure by independently resampling every number token from its marginal distribution.
As predicted by the universality observed in \S\ref{sec:problem_setup}, the periodic frequency distribution of number tokens is sufficient to produce Fourier spikes, and no co-occurrence information is needed.

\begin{table}[t]
\centering
\small
\caption{Data perturbation used in \S\ref{ssec:data}. All models on the 10B tokens from FineWeb-Edu.}
\label{tab:data_configs}
\rowcolors{1}{white}{blue!4}
\begin{tabularx}{\textwidth}{l >{\hsize=1.2\hsize}X >{\hsize=0.8\hsize}X}
\toprule
\textbf{Configuration} & \textbf{Perturbation} & \textbf{Structure of Data Removed} \\
\midrule
\texttt{Original} & $-$ & $-$ \\
\texttt{Isolate-}$k$ & Each sequence contains at most $k$ numbers token via packing. We use $k = 1, 2,$ and $8$. & Interaction across numbers. $k = 1$ indicates no interaction. \\
\texttt{ContextLength-}$\ell$ & Sequences split into windows of $\ell$. We use $\ell = 2, 4, 8,$ and $64$ & Role of broad context \\
\texttt{Swap Numbers} & Number token sequence replaced by that of another sequence, keeping number $n$-gram & Number $\leftrightarrow$ text association \\
\texttt{Unigram Replace} & Every number token resampled i.i.d.\ from marginal distribution to replace the original & Co-occurrence v.s.\ frequency \\
\bottomrule
\end{tabularx}
\end{table}

\begin{figure}
    \centering
    \includegraphics[width=\linewidth]{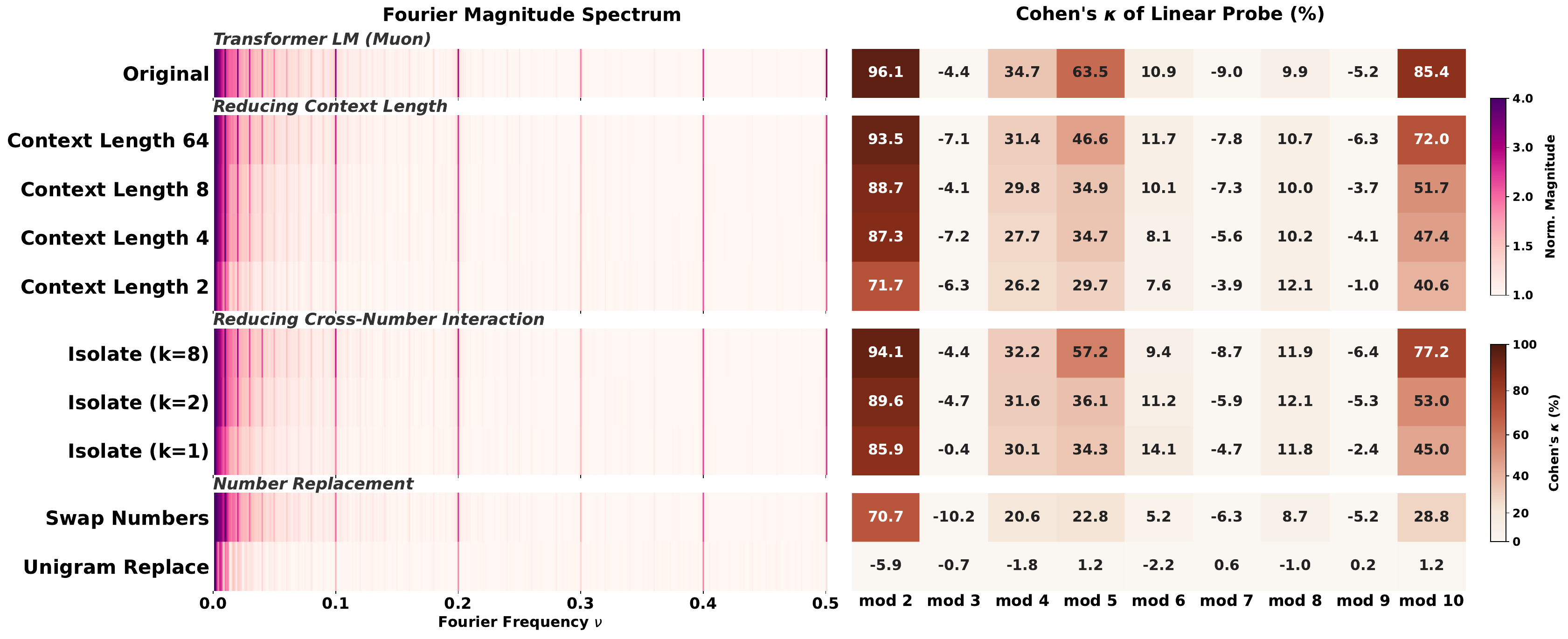}
    \caption{\textbf{Spectral convergence is universal but geometric convergence depends on the data signal.} 
    (\textit{Left}) Fourier spectra of Transformer embeddings trained under data perturbations in \Cref{tab:data_configs}. 
    All perturbations produce similar spikes to the original at periods $T = 2, 5, 10$, including \texttt{Unigram Replace}, which destroys all co-occurrence structure among number tokens. 
    (\textit{Right}) Cohen's $\kappa$ of linear probes for mod-$T$ classification tells a different story: \texttt{Original}, \texttt{Isolate-8}, and \texttt{Context Length 64} achieve strong probing, but with shorter context length or fewer numbers within each sequence (e.g. \texttt{Isolate-2}) is much weaker at $T = 5$ and $10$. \texttt{Swap Numbers} drops substantially, and \texttt{Unigram Replace} falls to chance.}
    \label{fig:data}
\end{figure}

\paragraph{Geometric convergence draws on several complementary data signals.}
The right panel of \Cref{fig:data} reveals that geometric convergence degrades gradually as different types of co-occurrence information are removed. The probing power is measured with Cohen's $\kappa$ for balanced classes, the metric that removes random guessing from accuracy for $T$-way classifications defined as $\kappa = (\mathrm{Accuracy} - \frac{1}{T}) / (1 - \frac{1}{T})$. When $\kappa = 0$, the probe is at chance and $\kappa = 100\%$ is perfect classification regardless of $T$.
\texttt{Swap Numbers}, which preserves number $n$-gram statistics but destroys the association between specific numbers and their text contexts, drops probing from $\kappa = 85.4$ to $28.8$ at $T = 10$, making text-number co-occurrence an important signal. But it is not the only one.

Longer context provides a second signal.
At \texttt{Context Length 2}, where each token sees only one neighbor, mod-10 probing already reaches $\kappa = 40.6$, well above \texttt{Swap Numbers} ($28.8$).
Increasing the window to $\ell = 4, 8,$ and $64$ steadily improves $\bmod 10$ probes ($\kappa = 47.4, 51.7,$ and $72.0$ respectively), showing that the model accumulates richer co-occurrence statistics from broader context.

Cross-number interaction provides an additional signal as well.
\texttt{Isolate}-$k$ directly controls cross-number interaction by restricting each packed sequence to contain at most $k$ number tokens. 
The extreme limit of $k=1$ isolates text-number co-occurrence completely, ensuring no two number tokens can interact within the same attention window.
Under this setting, it achieves $\kappa = 45.0$ at $T = 10$ and $85.9$ at $T = 2$.
Notably, even $k=1$ with a Transformer surpasses PPMI ($\kappa=27.1$) and word2vec ($\kappa=29.3$), suggesting that autoregressive language modeling with text-number co-occurrence alone extracts richer modular structure than classical embedding methods.
Allowing more numbers to co-occur within each attention block improves probing: $\kappa = 53.0$ at $k = 2$ and $77.2$ at $k = 8$ for $T = 10$.
The fact that \texttt{Isolate} ($k = 1$) still outperforms \texttt{Swap Numbers} confirms that text-number co-occurrence alone provides strong signal, but the gap to \texttt{Original} almost closed by \texttt{Isolate} ($k = 8$) shows that cross-number interaction contributes on top of it.
In all cases, Fourier spikes are fully preserved while probing degrades, reinforcing the two-tiered hierarchy between spectral and geometric convergence and they are driven by different mechanisms. These differences matter downstream: finetuned on 6- and 9-digit addition,
the \texttt{Original} checkpoint converges markedly faster than \texttt{Unigram Replace}
(\S\ref{app:downstream}).

\paragraph{Probing accuracy varies sharply across moduli.}
Across all conditions that achieve geometric convergence, the probing accuracy depends strongly on the modulus.
Mod 2, 5, and 10 are consistently the easiest ($\kappa = 96.1, 63.5, 85.4$ for \texttt{Original}), mod 4 achieves nontrivial probing ($\kappa = 34.9$), while moduli that share no common factor with 10 (e.g., 3, 7, 9) remain near chance.
This pattern is stable across all perturbations that preserve geometric convergence.
In \Cref{sec:addition}, we show that this structure can be traced to the tokenizer.

\subsection{Structural Attribution to Architecture and Optimizer} \label{ssec:arch_opt}

We now fix the pretraining data, and vary the architecture and optimizer. We train a 300M parameter Transformer model and two linear RNNs (Gated DeltaNet and Mamba-2 \citep{mamba2}) whose architectural and training details are shown in \S~\ref{app:model_training}. We vary two different optimizers for training LLMs: AdamW and Muon. Figure~\ref{fig:arch_opt} shows the results.

\begin{figure}[t]
    \centering
    \includegraphics[width=\linewidth]{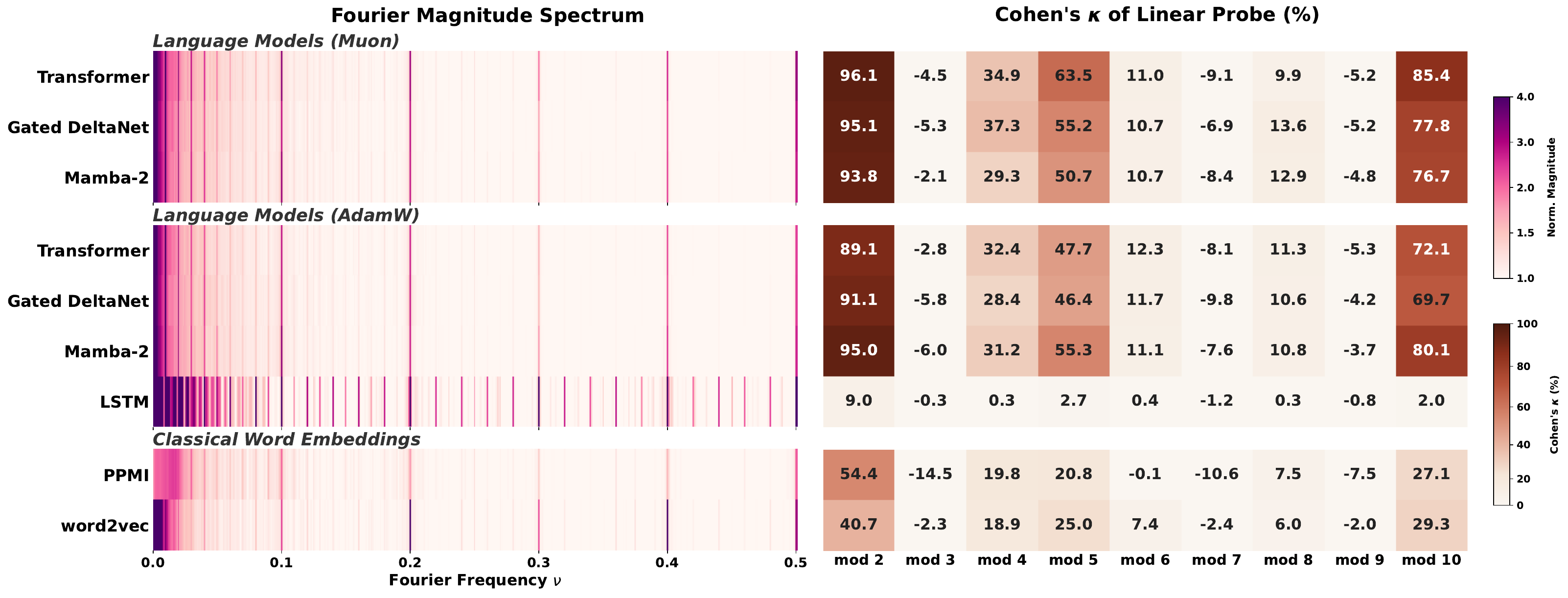}
    \caption{\textbf{Geometric convergence depends on architecture and optimizer.}
    (\textit{Left}) Fourier magnitude spectra of number embeddings across architectures and optimizers; all exhibit spikes at periods $T = 2, 5, 10$, including the \texttt{LSTM} and classical word embeddings.
    All three architectures produce similar Fourier spectra under both optimizers.
    (\textit{Right}) Mod-$T$ probe accuracy separates the models into tiers.
    With both Muon and AdamW, \texttt{Transformer}, \texttt{Gated DeltaNet}, and \texttt{Mamba-2} all achieve strong geometric convergence, while the \texttt{LSTM} remains near zero.
    Muon outperforms AdamW for \texttt{Transformer} and \texttt{Gated DeltaNet}, but \texttt{Mamba-2} with AdamW slightly outperforms its Muon counterpart.
    \texttt{PPMI} and \texttt{word2vec} fall in between.}
    \label{fig:arch_opt}
\end{figure}

\paragraph{Transformers and linear RNNs achieve geometric convergence; LSTMs do not.}
Under both Muon and AdamW, \texttt{Transformer}, \texttt{Gated DeltaNet}, and \texttt{Mamba-2} all achieve strong probing, with the Transformer performing best.
All three produce nearly identical Fourier spectra regardless of the optimizer, confirming that spectral convergence is architecture-independent and optimizer-independent in language pretraining task.
The 12-layer \texttt{LSTM}, trained with AdamW, develops even more prominent Fourier spikes but its probing accuracy remains near chance across all moduli. 
A 4-layer LSTM shows no improvement nor degradation, indicating that the failure is architectural rather than a matter of capacity (see \Cref{fig:lstm}).
We note that as shown in \Cref{fig:fourier_no_probe}, the Fourier spectrum of the LSTM embeddings closely resembles that of the number token marginal distribution, suggesting that the LSTM embeddings capture little beyond unigram frequency statistics for numbers. Surprisingly, a simple fix of adding residual connections could help LSTM learn geometric convergence (see \Cref{fig:lstm_add_residual}, \S~\ref{app:lstm}). 
Classical word embeddings fall in between: \texttt{PPMI} and \texttt{word2vec}, trained on the same 10B tokens, achieve moderate probing ($\kappa = 27.1$ and $29.3$ at $T = 10$) with clear Fourier spikes yet weaker probing, illustrating the dissociation between spectral and geometric convergence.

\paragraph{The effect of optimizer is architecture-dependent.}
Comparing the top and middle blocks of Figure~\ref{fig:arch_opt} isolates the effect of the optimizer.
Muon produces stronger probing for the Transformer ($\kappa = 85.4$ vs $72.1$ at $T = 10$) and for Gated DeltaNet ($77.8$ vs $69.7$), but Mamba-2 trained with AdamW actually outperforms Mamba-2 trained with Muon ($80.1$ vs $76.7$).
We find that the Transformer trained with Muon has the best probing 
performance. 
The optimizer's effect on geometric convergence thus depends on the architecture, and we do not observe a universal advantage for either optimizer in language pretraining task.


\paragraph{Spectral and geometric convergence co-emerge gradually.} \Cref{fig:dynamics_transformer_original} tracks $\Phi_T$ and probe accuracy throughout Transformer pretraining for  $T = 2, 5, 10$. Both increase smoothly with no phase transition, unlike the  grokking observed in modular arithmetic \citep{nanda2023progress}. We will discuss more on model behavior when trained directly on arithmetic in \Cref{sec:addition}.
\section{Convergent Evolution in Training on Arithmetic} \label{sec:addition}

Sections~\ref{sec:problem_setup} and~\ref{sec:pretraining}
studied models trained on general language, where Fourier features emerge from texts. 
We now ask whether convergent evolution also occurs when models are trained directly on arithmetic, where the prior given by human language is absent.

\paragraph{Experimental setup.}
We train 300M Transformers from random initialization (architecture of \S\ref{sec:problem_setup}) on examples $a+b=c$, masking the loss on prompt tokens up to the $=$ sign, for 3B tokens with two seeds each under Muon and AdamW. \emph{9-digit addition} samples operands of 1--9 digits (stratified by digit count), so operands span multiple number tokens.
\emph{3-digit addition} enumerates all pairs $a,b \in [0,999]$ with $a+b \leq 999$, so every operand and sum is a single token and 3B tokens amounts to roughly 1{,}000 epochs. Beyond mod-$T$ probes, we also train circular probes that project embeddings onto the unit circle (\Cref{fig:circular_probe}).


\begin{figure}[t]
    \centering
    \includegraphics[width=\linewidth]{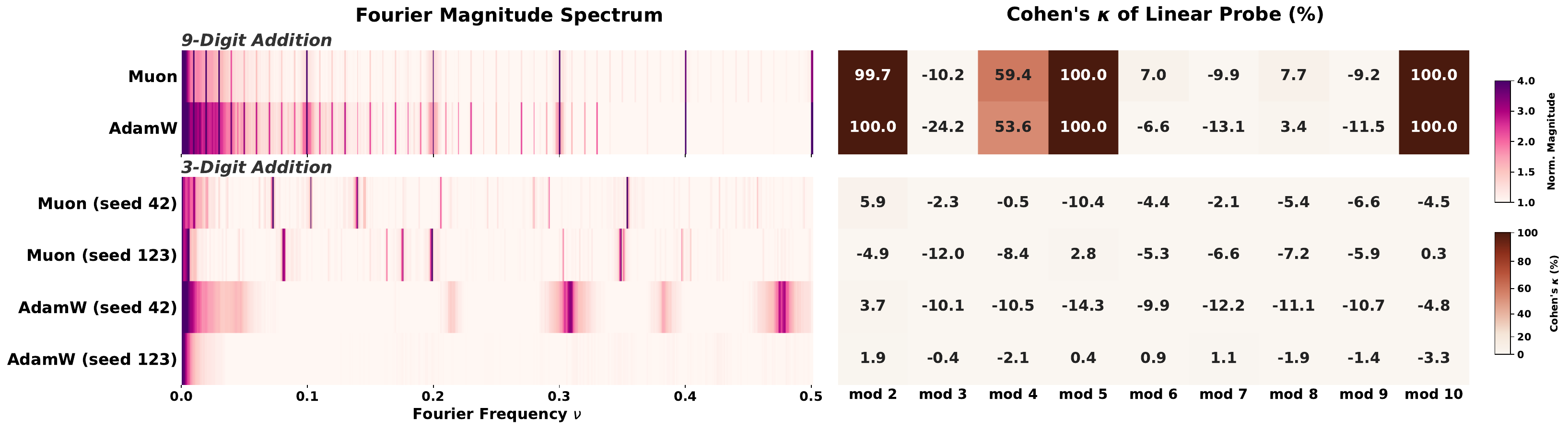}
    \caption{\textbf{Tokenization determines convergence in arithmetic.}
    (\textit{Top}) In 9-digit addition, both Muon and AdamW converge to the same spectral structure with sharp Fourier peaks and near-perfect $\kappa$ for mod 2, 5, and 10, showing both spectral and geometric convergence.
    (\textit{Bottom}) In 3-digit addition, where every operand and sum fits in a single token, the Fourier spectra vary across optimizers and random seeds, and $\kappa$ remains near chance for all moduli.
    Multi-token tokenization forces modular subproblems that produce convergent representations; single-token tokenization leaves the representation unconstrained.}
    \vspace{-1ex}
    \label{fig:addition}
\end{figure}

\paragraph{9-digit addition: convergence across optimizers.}
Both Muon and AdamW converge to the same spectral structure (\Cref{fig:addition}, top), with sharp Fourier peaks at the expected harmonics and near-perfect $\kappa$ for mod 2, 5, and 10.
The multi-token setting imposes constraints strong enough for both spectral and geometric convergence.

\paragraph{3-digit addition: no convergence without modular pressure.} The single-token setting behaves differently (\Cref{fig:addition}, bottom): with only 500K unique pairs, the learned representations vary across optimizer and random seed. Muon develops Fourier peaks at frequencies unaligned with modular periods, and $\kappa$ stays near chance. AdamW groks under one seed but never generalizes under the other (\Cref{fig:addition_dynamics}). Unlike mod-113 training \citep{nanda2023progress}, where modular structure is explicit, single-token addition imposes no modular constraint: the sequences $a+b=c$ are identical whether read mod-1000 or mod-1111, so convergent evolution does not occur and the representation is seed-dependent.

\paragraph{Why tokenization determines learned representation.}
In 9-digit addition, tokenization $[a_2, a_1, a_0] + [b_2, b_1, b_0] = [c_2, c_1, c_0]$ makes each output token satisfy $c_i = (a_i + b_i + \gamma_i) \bmod 1000$, with carry $\gamma_i \in \{0, 1\}$ and $\gamma_0 = 0$. Every output position is therefore a mod-1000 classification problem, especially the least significant position with no carry. With tied embeddings (\Cref{tab:config}), the output logits depend directly on the embedding matrix, pressuring it toward high $\Phi_{1000}$ and hence non-trivial $\Phi_T$ for all $T \mid 1000 = 2^3 \cdot 5^3$. Single-token addition imposes no such constraint, so $\Phi_T$ depends on the optimizer and seed. This is a second route through the two-tiered hierarchy: multi-token tokenization creates modular subproblems that force both spectral and geometric convergence, while single-token tokenization guarantees neither. See more ablations in \Cref{app:tokenizer}.


\paragraph{Non-decimal addition: the base selects the moduli.} Are the dominant periods 2, 5, 10 an artifact of base-10? We train on 6-digit addition in base-10, base-15, and base-16 (setup in \S\ref{app:nondecimal}). Base-10 training yields perfect probes at mod 2, 5, and 10, and base-15 at mod 3, 5, and 15 (\Cref{fig:non_decimal}), consistent with the Chinese Remainder Theorem: the tasks provide an inductive bias toward $\mathbb{Z}/10\mathbb{Z} \cong \mathbb{Z}/2\mathbb{Z} \times \mathbb{Z}/5\mathbb{Z}$ and $\mathbb{Z}/15\mathbb{Z} \cong \mathbb{Z}/3\mathbb{Z} \times \mathbb{Z}/5\mathbb{Z}$, respectively.
Base-16 lacks this structure: $16=2^4$ has no coprime factorization, so mod-16 addition cannot be split into independent smaller cyclic subproblems the way mod-10 splits into mod-2 and mod-5. Fourier peaks land at arbitrary frequencies and probes stay at chance. Under arithmetic pressure, convergent evolution thus inherits its periods from the numeral system.

\vspace{-1ex}
\section{Conclusion and Discussion} \label{sec:conclusion}

We have shown that periodic number representations in language models exhibit a two-tiered convergence: Fourier spikes are universal, but linearly separable mod-$T$ classes emerge only when data, architecture, and optimizer align.
The central lesson is that visible structure in representations does not guarantee functional organization: the LSTM and even the raw token distribution develop more prominent Fourier spikes than the Transformer yet achieve chance-level probing.

More generally, any representation-level diagnostic could mistake statistical artifacts of the training distribution for learned structure.
Our controlled perturbation approach offers a complementary lens to instance-level attribution methods such as influence functions \citep{koh2017influence}: rather than attributing predictions to individual training examples, we attribute learned representations to structural properties of the data distribution.

Analogous periodic representations have been found for days of the week and months of the year \citep{engels2025not, karkada2026symmetry}; whether the spectral-geometric dissociation extends to these and other cyclic concepts is a natural next step.
More broadly, the spectral-geometric hierarchy introduced here provides a concrete framework for distinguishing superficial from functional feature learning.
This distinction may prove important well beyond numerical representations as we increasingly rely on representation-level diagnostics to understand large language models.
\section*{Acknowledgments}
The authors acknowledge the Center for Advanced Research Computing (CARC) at the University of Southern California for providing computing resources that have contributed to the research results reported within this publication. We also acknowledge the use of the USC NLP cluster provided by the USC NLP Group. 
DF and RJ were also supported by a gift from the USC-Capital One Center for Responsible AI and Decision Making in Finance (CREDIF). 
RJ was supported in part by the National Science Foundation under Grant No. IIS-2403436. VS was supported by National Science Foundation award CCF-2239265, an Amazon Research Award, a Google Research Scholar Award and
an Okawa Foundation Research Grant.
The work was done in part while some of the authors were visiting the Simons Institute for the Theory of Computing.
This work used the Delta system at the National Center for Supercomputing Applications through allocation CIS250737 from the Advanced Cyberinfrastructure Coordination Ecosystem: Services \& Support (ACCESS) program, which is supported by National Science Foundation grants \#2138259, \#2138286, \#2138307, \#2137603, and \#2138296. 
This work was supported in part by the NVIDIA Academic Grant Program.
The GPU resources provided by NVIDIA were essential for training and analyzing the
300M-parameter models examined in this study.
Any opinions, findings, and conclusions or recommendations expressed in
this material are those of the author(s) and do not reflect the views of the funding agencies.


\bibliography{main}
\bibliographystyle{colm2026_conference}

\clearpage
\appendix
\section*{Appendix}
\appendix
\startcontents[appendix]
\printcontents[appendix]{l}{1}{\setcounter{tocdepth}{2}}

\clearpage
\section{\texorpdfstring{Proof of \Cref{thm:fourier}}{Proof of Theorem 1}}
\label{app:proof-fourier}
We first restate the theorem here.
\maintheorem*
We first start with two lemmas and then prove them in parts.

\subsection{Lemmas: Fourier-variance identity}

We first establish a lemma connecting the class means $\vmu_r$ to the Fourier coefficients $\vF_\nu$.

\begin{lemma}
\label{lem:class-mean-dft}
For each $\ell = 0, \ldots, T-1$, define
\[
\hat{\vmu}[\ell] \;=\; \frac{1}{\sqrt{T}} \sum_{r=0}^{T-1} \vmu_r \, e^{-2\pi i \ell r / T}.
\]
Then $\hat{\vmu}[\ell] = \sqrt{T/N} \cdot \vF_{\ell/T}$.
\end{lemma}

\begin{proof}
Substituting $\vmu_r = \frac{1}{|C_r|}\sum_{n \in C_r} \ve(n) = \frac{T}{N}\sum_{n \in C_r} \ve(n)$
(since $|C_r| = N/T$ when $T | N$.) and pulling the constant into the outer sum:
\begin{align}
\hat{\vmu}[\ell]
&= \frac{\sqrt{T}}{N} \sum_{r=0}^{T-1} \sum_{n \in C_r} \ve(n) \, e^{-2\pi i \ell r / T}.
\end{align}
Every $n \in \{0,\ldots,N-1\}$ belongs to exactly one class $C_r$ with $r = n \bmod T$,
so we can write $n = mT + r$ for some non-negative integer $m$.
Then $e^{-2\pi i \ell r/T} = e^{-2\pi i \ell n/T}$, since the additional $m\ell$ full periods contribute
an integer multiple of $2\pi$ to the exponent.
Re-indexing the double sum as a single sum over $n$:
\begin{align}
\hat{\vmu}[\ell]
&= \frac{\sqrt{T}}{N} \sum_{n=0}^{N-1} \ve(n) \, e^{-2\pi i \ell n / T}.
\end{align}
Because $T \mid N$, the ratio $\ell/T$ is one of the $N$ Fourier frequencies,
so comparing with $\vF_{\ell/T} = \frac{1}{\sqrt{N}} \sum_{n} \ve(n) \, e^{-2\pi i (\ell/T) n}$ gives
\[
\hat{\vmu}[\ell]
= \frac{\sqrt{T}}{N} \cdot \sqrt{N} \cdot \vF_{\ell/T}
= \sqrt{\frac{T}{N}} \cdot \vF_{\ell/T}. \qedhere
\]
\end{proof}

Next, we state our second lemma on the identities for $\Tr(\mS_B)$ and $\Tr(\mS_W)$.

\begin{lemma}\label{lem:fourier-variance-identity}
    Let the between-class and within-class variances $\mS_B$ and $\mS_W$ be defined as in \Cref{thm:fourier}, they satisfy
\[
\Tr(\mS_B) = \frac{\Phi_T}{N}, \qquad \Tr(\mS_W) = \frac{1}{N}\sum_{\nu \notin H_T}\|\vF_\nu\|^2.
\]
\end{lemma}
\begin{proof}
Define the $T \times d$ matrix $\mM$ whose rows are the class means:
\[
\mM \;=\;
\begin{pmatrix}
\text{---}\; \vmu_0^\top \;\text{---} \\
\text{---}\; \vmu_1^\top \;\text{---} \\
\vdots \\
\text{---}\; \vmu_{T-1}^\top \;\text{---}
\end{pmatrix}
\in \mathbb{R}^{T \times d},
\]
and let $\mU$ be the $T \times T$ matrix with entries
$U_{\ell r} = \frac{1}{\sqrt{T}} e^{-2\pi i \ell r/T}$.
Define $\hat{\mM} = \mU \mM \in \mathbb{C}^{T \times d}$, whose rows are $\hat{\vmu}[\ell]^\top$:
\[
\hat{\mM} \;=\;
\begin{pmatrix}
\text{---}\; \hat{\vmu}[0]^\top \;\text{---} \\
\text{---}\; \hat{\vmu}[1]^\top \;\text{---} \\
\vdots \\
\text{---}\; \hat{\vmu}[T-1]^\top \;\text{---}
\end{pmatrix}
\in \mathbb{C}^{T \times d}.
\]

Since $\mU$ is the DFT matrix in $T$ dimensions, $\mU$ is unitary. 

Therefore,
\[
\sum_{\ell=0}^{T-1} \|\hat{\vmu}[\ell]\|^2
\;=\; \|\hat{\mM}\|_F^2
\;=\; \|\mU \mM\|_F^2
\;=\; \|\mM\|_F^2
\;=\; \sum_{r=0}^{T-1} \|\vmu_r\|^2.
\]

This is simply the fact that a unitary change of basis preserves the sum of squared norms.

For $\ell = 0$: $\hat{\vmu}[0] = \sqrt{T/N} \cdot \vF_0$.
Since $\vF_0 = \frac{1}{\sqrt{N}}\sum_n \ve(n) = \sqrt{N}\,\vmu$, we have
$\hat{\vmu}[0] = \sqrt{T}\,\vmu$, and therefore
$\frac{1}{T}\|\hat{\vmu}[0]\|^2 = \|\vmu\|^2$.

The between-class variance is:
\begin{align}
\Tr(\mS_B)
&= \frac{1}{T}\sum_r \|\vmu_r - \vmu\|^2
= \frac{1}{T} \left(\sum_r \|\vmu_r\|^2
  - 2\vmu^\top\!\left(\frac{1}{T}\sum_r \vmu_r\right) + \|\vmu\|^2\right) \\
&= \frac{1}{T}\left(\sum_r \|\vmu_r\|^2 - \|\vmu\|^2\right)
= \frac{1}{T}\left(\sum_{\ell=1}^{T-1}\|\hat{\vmu}[\ell]\|^2\right) \\
&= \frac{1}{T}\left(\sum_{\ell=1}^{T-1}\frac{T}{N}\|\vF_{\ell/T}\|^2\right)
= \frac{1}{N}\sum_{\ell=1}^{T-1}\|\vF_{\ell/T}\|^2
= \frac{\Phi_T}{N},
\end{align}
where the second line uses $\frac{1}{T}\sum_r \vmu_r = \vmu$,
which holds because the classes $\{C_r\}$ partition $\{0,\ldots,N-1\}$
into equal-sized groups.

For $\Tr(\mS_W)$, the total variance decomposes as $V = \Tr(\mS_B) + \Tr(\mS_W)$.
Define the $N \times d$ matrix $\mE$ whose rows are $\ve(0)^\top, \ldots, \ve(N-1)^\top$, and let $\mU_N$ be the $N \times N$ DFT matrix with entries $(\mU_N)_{kn} = \frac{1}{\sqrt{N}} e^{-2\pi i k n / N}$.
Then $\hat{\mE} = \mU_N \mE$ has rows $\vF_\nu^\top$, and since $\mU_N$ is unitary,
\[
\sum_{n=0}^{N-1} \|\ve(n)\|^2 = \|\mE\|_F^2 = \|\hat{\mE}\|_F^2 = \sum_{\nu} \|\vF_\nu\|^2.
\]
Since $\vF_0 = \sqrt{N}\,\vmu$, we have $\|\vF_0\|^2 = N\|\vmu\|^2$, and therefore
\[
V = \frac{1}{N}\sum_{n=0}^{N-1}\|\ve(n) - \vmu\|^2
  = \frac{1}{N}\left(\sum_{n}\|\ve(n)\|^2 - N\|\vmu\|^2\right)
  = \frac{1}{N}\sum_{\nu \neq 0}\|\vF_\nu\|^2.
\]

Therefore:
\[
\Tr(\mS_W) = V - \Tr(\mS_B)
= \frac{1}{N}\sum_{\nu \neq 0}\|\vF_\nu\|^2
  - \frac{1}{N}\sum_{\ell=1}^{T-1}\|\vF_{\ell/T}\|^2
= \frac{1}{N}\sum_{\nu \notin H_T}\|\vF_\nu\|^2,
\]
where $H_T = \{0,\, 1/T,\, 2/T,\, \ldots,\, (T{-}1)/T\}$ collects the
zero-frequency and harmonic frequencies.
\end{proof}

\subsection{Necessary condition}

\begin{proof}[Proof of Part (i) of \Cref{thm:fourier}]
If $\Phi_T = 0$, then $\Tr(\mS_B) = 0$ by \Cref{lem:fourier-variance-identity}.
Since $\mS_B$ is positive semidefinite, $\Tr(\mS_B) = 0$ implies $\mS_B = \bm{0}$, which requires $\vmu_r = \vmu$ for all $r = 0, \ldots, T-1$.
When all class means coincide, the class-conditional distributions of $\ve(n)$ share the same first moment, so no linear probe (or any probe relying on mean separation) can distinguish the $T$ classes above the chance rate $1/T$.
\end{proof}

\subsection{Insufficiency condition} \label{app:insufficient}


\begin{proof}[Proof of Part (ii)]
We construct embeddings whose residue classes interleave periodically on the real line, then show that the geometry of linear decision boundaries prevents any $T$-class linear classifier from exceeding chance-level accuracy by more than $\varepsilon$.
Fix $T \ge 2$, $C > 0$, and $\varepsilon > 0$.
Set $K = \lceil (T{-}1)/(T\varepsilon) \rceil$ and $N = KT$.
Every index $n \in \{0, \dots, N{-}1\}$ has a unique decomposition $n = mT + r$ with residue $r \in \{0, \dots, T{-}1\}$ and block index $m \in \{0, \dots, K{-}1\}$.
Define
\[
  e(n) \;=\; A \cdot \underbrace{(n \bmod T)}_{\text{residue } r}
  \;+\; B \cdot \underbrace{\left\lfloor \frac{n}{T} \right\rfloor}_{\text{block index } m},
\]
with $A, B > 0$ to be chosen.
Intuitively, $A$ controls $n \bmod T$: enlarging $A$ pulls numbers sharing the same residue class $C_r$ together.
The parameter $B$ controls $\lfloor n/T \rfloor$, which measures how many full copies of $T$ fit below $n$: increasing $B$ clusters numbers of similar magnitude together, regardless of their residue.
We now show how $A$ determines the Fourier power while $B$ controls the interleaving that defeats linear classifiers.

\paragraph{Fourier power.}
Within class $C_r = \{r, r{+}T, \dots, r{+}(K{-}1)T\}$, the term $Ar$ is constant and only the block term varies, so the class mean is
\[
  \mu_r
  = \frac{1}{K}\sum_{m=0}^{K-1}(Ar + Bm)
  = Ar + \frac{B(K{-}1)}{2}.
\]
The grand mean is $\mu = A(T{-}1)/2 + B(K{-}1)/2$, giving $\mu_r - \mu = A\bigl(r - (T{-}1)/2\bigr)$.
Note that the block term $B \lfloor n/T \rfloor$ contributes nothing to $S_B$: its class mean $B(K{-}1)/2$ is identical across all classes and cancels in $\mu_r - \mu$.
Therefore
\[
  S_B
  = \frac{1}{T}\sum_{r=0}^{T-1}(\mu_r - \mu)^2
  = \frac{A^2}{T}\sum_{r=0}^{T-1}\Bigl(r - \frac{T{-}1}{2}\Bigr)^{\!2}
  = \frac{A^2(T^2 - 1)}{12},
\]
where the last equality uses the standard identity for the second central moment of $T$ consecutive integers.
By \Cref{lem:fourier-variance-identity}, $\Phi_T = N \cdot S_B = A^2 KT(T^2-1)/12$.
Setting $A = \sqrt{12C / (KT(T^2 - 1))}$ gives $\Phi_T = C$.

\paragraph{Periodic interleaving.}
Choose $B > (T{-}1)A$ so that consecutive blocks separate on the real line.
The largest value in block $m$ is $e((T{-}1) + mT) = (T{-}1)A + Bm$, and
the smallest value in block $m{+}1$ is $e(0 + (m{+}1)T) = B(m{+}1)$.
Since $B > (T{-}1)A$, we have $B(m{+}1) > (T{-}1)A + Bm$, so block $m$
and block $m{+}1$ occupy disjoint intervals on the real line.
Within each block, the $T$ points are sorted by residue:
$e(mT) = Bm < e(mT{+}1) = A + Bm < \cdots < e(mT{+}T{-}1) = (T{-}1)A + Bm$,
since $e(n)$ is increasing in $n \bmod T$ for fixed $\lfloor n/T \rfloor$.
Combining both observations, sorting all $N$ embeddings by value yields the
natural ordering: the $i$-th smallest embedding ($0$-indexed) is
\[
  x_{(i)} \;=\; e(i) \;=\; A \cdot (i \bmod T) \;+\; B \cdot \lfloor i/T \rfloor,
  \qquad i = 0, 1, \dots, N{-}1,
\]
whose class label is $i \bmod T$.
The sorted class label sequence is therefore $(0, 1, \dots, T{-}1)$ repeated
$K$ times:
\[
  \underbrace{0, 1, \dots, T{-}1}_{\text{block } 0},\;\;
  \underbrace{0, 1, \dots, T{-}1}_{\text{block } 1},\;\;
  \dots,\;\;
  \underbrace{0, 1, \dots, T{-}1}_{\text{block } K{-}1}.
\]
Any contiguous subsequence of length $T$ or longer contains at least one
complete cycle and therefore includes at least one point from every class.

\paragraph{Classification bound.}
A $T$-class linear classifier assigns each instance $x \in \mathbb{R}$ to
$\operatorname{argmax}_{c \in \{0,\dots,T-1\}} (w_c x + b_c)$ for
parameters $w_c, b_c \in \mathbb{R}$. This is the hypothesis class of
multiclass logistic regression.
Note that this hypothesis class is expressive enough to perfectly classify
$T$ contiguous groups: when $B$ is small, the $T$ classes cluster near
$0, A, 2A, \dots, (T{-}1)A$, and setting $w_c = c$ with biases cascaded
via $b_0 = 0$, $b_{c+1} = b_c - (cA + A/2)$ places the $T{-}1$ decision
boundaries between consecutive clusters, achieving $100\%$ accuracy.
Next we exploit the limitation that linear classifiers partition $\mathbb{R}$ into at most $T$ contiguous intervals: any two
of the $T$ lines $x \mapsto w_c x + b_c$ intersect in at most one point,
yielding at most $T - 1$ breakpoints and hence at most $T$ intervals, each
assigned to a single class.
Now consider any such interval.
Each complete cycle $(0, 1, \dots, T{-}1)$ contained in the interval contributes exactly one point from every class, so the assigned label matches exactly a $1/T$ fraction of those points.
Only at interval boundaries can a partial cycle contribute additional correct predictions: at most one per boundary, for a total of at most $T - 1$ extra correct points across all $T - 1$ boundaries.
It follows that
\[
  \mathrm{accuracy}
  \;\le\; \frac{N/T + T - 1}{N}
  \;=\; \frac{1}{T} + \frac{T-1}{KT}.
\]
The choice $K = \lceil (T{-}1)/(T\varepsilon) \rceil$ guarantees $(T{-}1)/(KT) \le \varepsilon$, completing the proof.

\Cref{fig:thm_part_ii_examples,fig:part_ii_T_10} illustrates this construction for $T = 5,N = 25$, and for $T = 10, N = 1000$, where in both cases $\varepsilon$ achieves its minimum at $\frac{T-1}{N}$, at 16\% and 0.9\% respectively.
\end{proof}

\begin{figure}
    \centering
    \includegraphics[width=\linewidth]{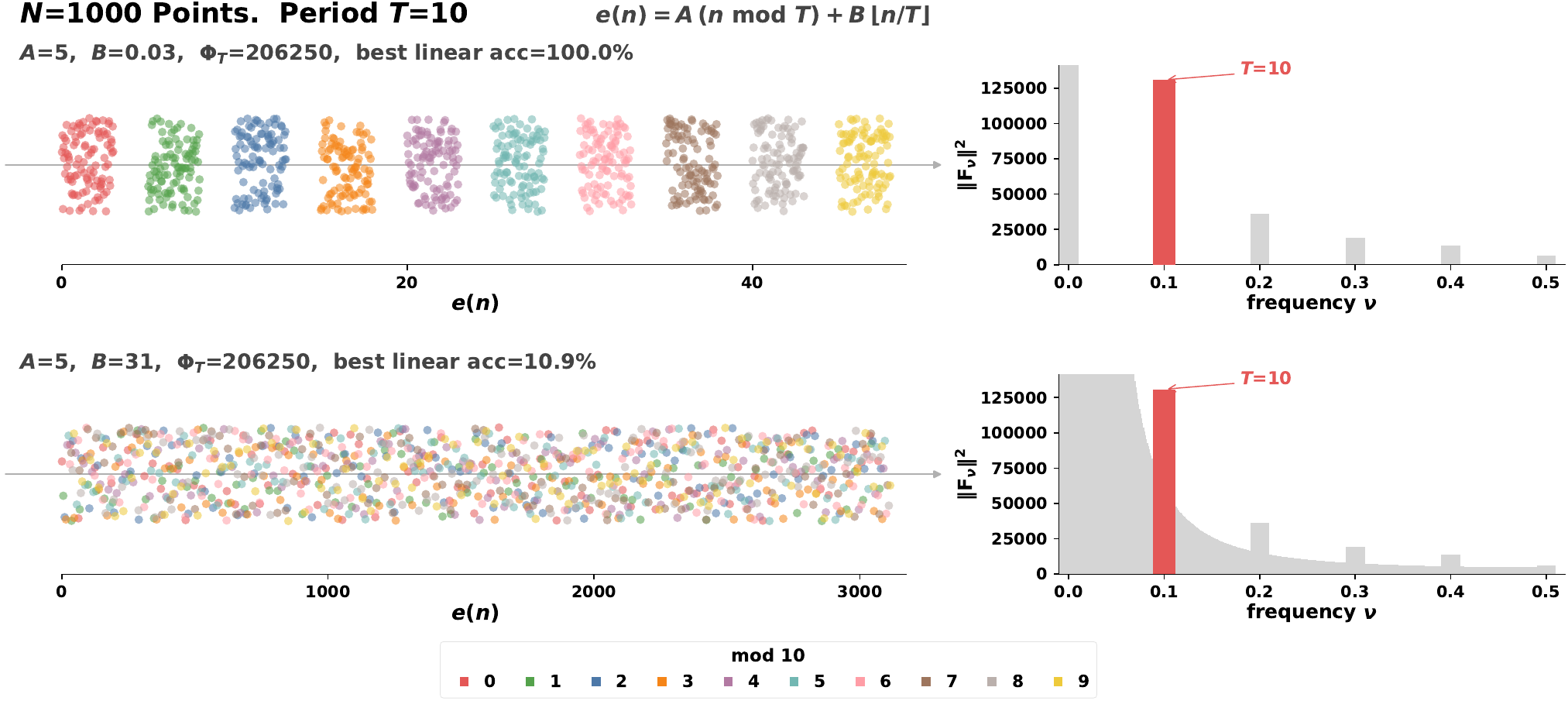}
    \caption{Emergence of Fourier structure in a constructed embedding $e(n) = A(n \bmod T) + B\lfloor n/T \rfloor$ with $T=10$, $A=5$, $N=1000$. Each dot is a number $n$ placed at its scalar embedding value on the horizontal axis and colored by its class $n \bmod T$; since $e(n)$ is one-dimensional, the vertical coordinate carries no information and is jittered purely for visibility so that overlapping points of different classes remain distinguishable. \textbf{Top ($B=0.03$)}: within-block drift is small relative to between-class spacing, so the sorted line separates into 10 clean color bands: a linear classifier on $e(n)$ recovers the mod-10 residue at near-perfect accuracy, and the DFT concentrates energy at the fundamental frequency $\nu = 1/T = 0.1$. \textbf{Bottom ($B=31$)}: the between-block drift $B$ dominates, interleaving classes along the line so that every mod-10 lane spans the full range; best linear accuracy collapses to $\approx 1/T + \varepsilon$ with $\varepsilon = 0.9$\%. The Fourier spectrum keeps the same peak at $\nu=0.1$ (same $\Phi_T$), showing that periodic energy at the fundamental is a property of the construction itself, independent of whether class identity is linearly decodable.}
    \label{fig:part_ii_T_10}
\end{figure}

The token frequency distribution in \Cref{fig:fourier_no_probe} provides an empirical counterexample: LSTM exhibits clear Fourier spikes at $T = 2, 5, 10$ yet achieves chance-level probing for all moduli, showing that the construction above captures a phenomenon that occurs in practice.

\subsection{Lower and Upper Bounds} \label{app:bounds}
\begin{lemma}
    Assume $d \geq T - 1$ and $\mS_W$ is invertible (which holds when $d \leq N - T$, or after projecting to a subspace of dimension at most $N - T$).
In Fisher's Linear Discriminant Analysis \citep{fisher1936lda}, the optimal linear probe maximizes the ratio of between-class to within-class variance along a projection direction.
The separability of this optimal discriminant is characterized by $\lambda_{\max}(\mS_W^{-1} \mS_B)$, the largest generalized eigenvalue of the scatter matrix pair, which satisfies
\[
\frac{1}{(T-1) \cdot \mathrm{cond}(\mS_W) } \cdot \frac{\Phi_T}{N \cdot \lambda_{\min}(\mS_W)} \;\leq\; \lambda_{\max}(\mS_W^{-1} \mS_B) \;\leq\; \frac{\Phi_T}{N \cdot \lambda_{\min}(\mS_W)}.
\]
The ratio between the upper and lower bounds is $(T-1) \cdot \mathrm{cond}(\mS_W)$, where $\mathrm{cond}(\mS_W) = \lambda_{\max}(\mS_W) / \lambda_{\min}(\mS_W)$.
\end{lemma}
\begin{proof}
The optimal linear discriminant for $T$-class classification projects the data onto the direction maximizing the generalized Rayleigh quotient:
\[
\lambda_{\max}(\mS_W^{-1}\mS_B) = \max_{\vv \neq \vzero} \frac{\vv^\top \mS_B\, \vv}{\vv^\top \mS_W\, \vv}.
\]

\emph{Upper bound.}
For any $\vv \neq \vzero$, the numerator satisfies $\vv^\top \mS_B\, \vv \leq \lambda_{\max}(\mS_B)\|\vv\|^2 \leq \Tr(\mS_B)\|\vv\|^2$, where the second inequality holds because $\mS_B$ is positive semidefinite and $\lambda_{\max}(\mS_B) \leq \Tr(\mS_B)$.
The denominator satisfies $\vv^\top \mS_W\, \vv \geq \lambda_{\min}(\mS_W)\|\vv\|^2$.
Therefore:
\[
\lambda_{\max}(\mS_W^{-1}\mS_B) \leq \frac{\Tr(\mS_B)}{\lambda_{\min}(\mS_W)} = \frac{\Phi_T}{N \cdot \lambda_{\min}(\mS_W)}.
\]

\emph{Lower bound.}
The matrix $\mS_B$ has rank at most $\min(d, T-1)$; since $d \geq T - 1$ by assumption, this simplifies to $T - 1$.
This holds because the $T$ vectors $\{\vmu_r - \vmu\}_{r=0}^{T-1}$ satisfy $\sum_r (\vmu_r - \vmu) = \vzero$ and therefore span a subspace of dimension at most $T-1$.
It follows that $\lambda_{\max}(\mS_B) \geq \Tr(\mS_B)/(T-1) = \Phi_T/(N(T-1))$.
Let $\vv^*$ be the unit eigenvector of $\mS_B$ corresponding to $\lambda_{\max}(\mS_B)$.
Then:
\[
\lambda_{\max}(\mS_W^{-1}\mS_B) \geq \frac{{\vv^*}^\top \mS_B\, \vv^*}{{\vv^*}^\top \mS_W\, \vv^*} = \frac{\lambda_{\max}(\mS_B)}{{\vv^*}^\top \mS_W\, \vv^*} \geq \frac{\lambda_{\max}(\mS_B)}{\lambda_{\max}(\mS_W)} \geq \frac{\Phi_T}{N \cdot (T-1) \cdot \lambda_{\max}(\mS_W)}.
\]

\emph{Gap between the bounds.}
The ratio of the upper to lower bound is:
\[
\frac{\Phi_T / (N \cdot \lambda_{\min}(\mS_W))}{\Phi_T / (N \cdot (T-1) \cdot \lambda_{\max}(\mS_W))} = (T-1) \cdot \frac{\lambda_{\max}(\mS_W)}{\lambda_{\min}(\mS_W)} = (T-1) \cdot \mathrm{cond}(\mS_W).
\]
The Fourier power spectrum fully determines $\Phi_T$ via Part~(i), but $\mathrm{cond}(\mS_W)$ depends on the directional structure of within-class variation, which the power spectrum $\{\|\vF_\nu\|^2\}$ does not capture.
Specifically, $\|\vF_\nu\|^2$ aggregates power across all $d$ embedding dimensions at frequency $\nu$, discarding any information about which dimensions carry the periodic signal versus which carry within-class noise.
As a result, two embeddings with identical power spectra (and hence identical $\Phi_T$) but different within-class covariance structures can yield $\lambda_{\max}(\mS_W^{-1}\mS_B)$ values differing by up to a factor of $(T-1) \cdot \mathrm{cond}(\mS_W)$, producing vastly different probe accuracies.
\end{proof}
\clearpage
\section{Experiments} \label{app:exp}

\subsection{Model and Training Details} \label{app:model_training}
\begin{table}[h]
\centering
\caption{Model architectures and training configurations.}
\label{tab:config}
\small
\setlength{\tabcolsep}{4pt}
\begin{tabular}{l*{4}{c}}
\toprule
& Transformer & Gated DeltaNet & Mamba-2 & LSTM \\
\midrule
\multicolumn{5}{l}{\textit{Architecture}} \\
\quad Total parameters          & 320M   & 318M   & 316M   & 232M \\
\quad \; $\rightarrow$ Embedding              & 131M   & 131M   & 131M   & 131M \\
\quad \; $\rightarrow$ Non-embedding          & 189M   & 186M   & 185M   & 101M \\
\quad Tied embeddings           & \cmark & \cmark & \cmark & \cmark \\
\quad Layers                    & 12     & 12     & 28     & 12 \\
\quad Hidden dim $d$            & 1024   & 1024   & 1024   & 1024 \\
\quad Heads                     & 16 (8 KV) & 16  & 32     & --- \\
\quad Head dim                  & 64     & 64     & 64     & --- \\
\quad MLP intermediate          & 4096   & 4096   & ---    & --- \\
\quad MLP activation            & SwiGLU & SwiGLU & ---    & --- \\
\quad Positional encoding       & RoPE & None & None & None \\
\quad Normalization             & RMSNorm & RMSNorm & RMSNorm & --- \\
\quad Sequence mechanism        & GQA    & Linear Attn + Gate & SSM ($d_\text{state}{=}128$) & Recurrence \\
\quad Short convolution         & ---    & size 4 & size 4 & --- \\
\quad SSM expand factor         & ---    & $1.5\times$ (value) & $2\times$ & --- \\
\quad Dropout                   & ---      & ---      & ---      & 0.1 \\
\midrule
\multicolumn{5}{l}{\textit{Muon optimizer}} \\
\quad 2D weight LR              & \multicolumn{4}{c}{$3 \times 10^{-3}$, momentum $= 0.95$} \\
\quad Embed/norm/bias LR        & \multicolumn{4}{c}{$3 \times 10^{-4}$ (AdamW, $\beta_2{=}0.95$)} \\
\quad Weight decay              & \multicolumn{4}{c}{0.01 (2D weights only)} \\
\midrule
\multicolumn{5}{l}{\textit{AdamW optimizer}} \\
\quad Learning rate             & \multicolumn{4}{c}{$3 \times 10^{-4}$} \\
\quad $(\beta_1, \beta_2)$      & \multicolumn{4}{c}{$(0.9, 0.95)$} \\
\quad Weight decay              & \multicolumn{4}{c}{0.01} \\
\midrule
\multicolumn{5}{l}{\textit{Shared training}} \\
\quad Context length & \multicolumn{4}{c}{1024} \\
\quad Batch size                & \multicolumn{4}{c}{512 sequences (${\sim}$524K tokens/step)} \\
\quad LR schedule               & \multicolumn{4}{c}{Cosine decay, 500 warmup steps, min $= 10\%$ of peak} \\
\quad Training tokens           & \multicolumn{4}{c}{${\sim}$9.4B (1 epoch of FineWeb-Edu 10BT)} \\
\quad Precision                 & \multicolumn{4}{c}{bfloat16 mixed precision} \\
\bottomrule
\end{tabular}
\end{table}

\subsection{LSTM ablation} \label{app:lstm}
We ablate the number of layers of LSTMs and find reducing the number of layers to 4 instead of 12 does not change the phenomenon: both will learn fourier spikes but no probing performance. Both have huge condition number on $\mS_W$ yet huge $\Phi_T$ as well.

In \Cref{fig:lstm_add_residual}, we simply add residual connections to the 12-layer LSTM model and surprisingly, it's able to learn some extent of geometrical structures. It demonstrates the power of residual connections in model design. 

\begin{figure}[ht]
    \centering
    \includegraphics[width=\linewidth]{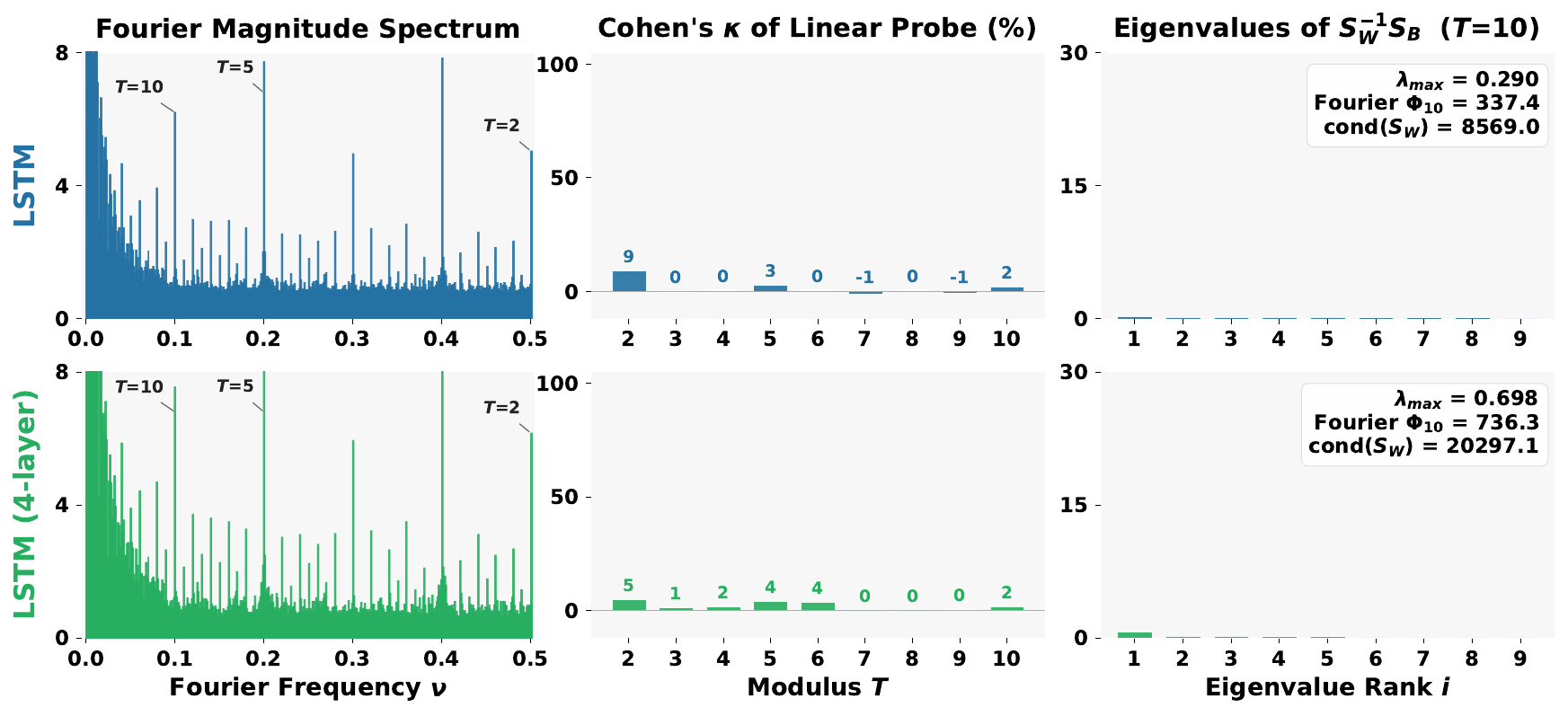}
    \caption{\textbf{Ablation on the depth of LSTM models.} We find that reducing the number of layers to 4 (green) instead of 12 (blue) does not change the phenomenon: both will learn Fourier spikes but no probing performance.}
    \label{fig:lstm}
\end{figure}

\begin{figure}[h]
    \centering
    \includegraphics[width=\linewidth]{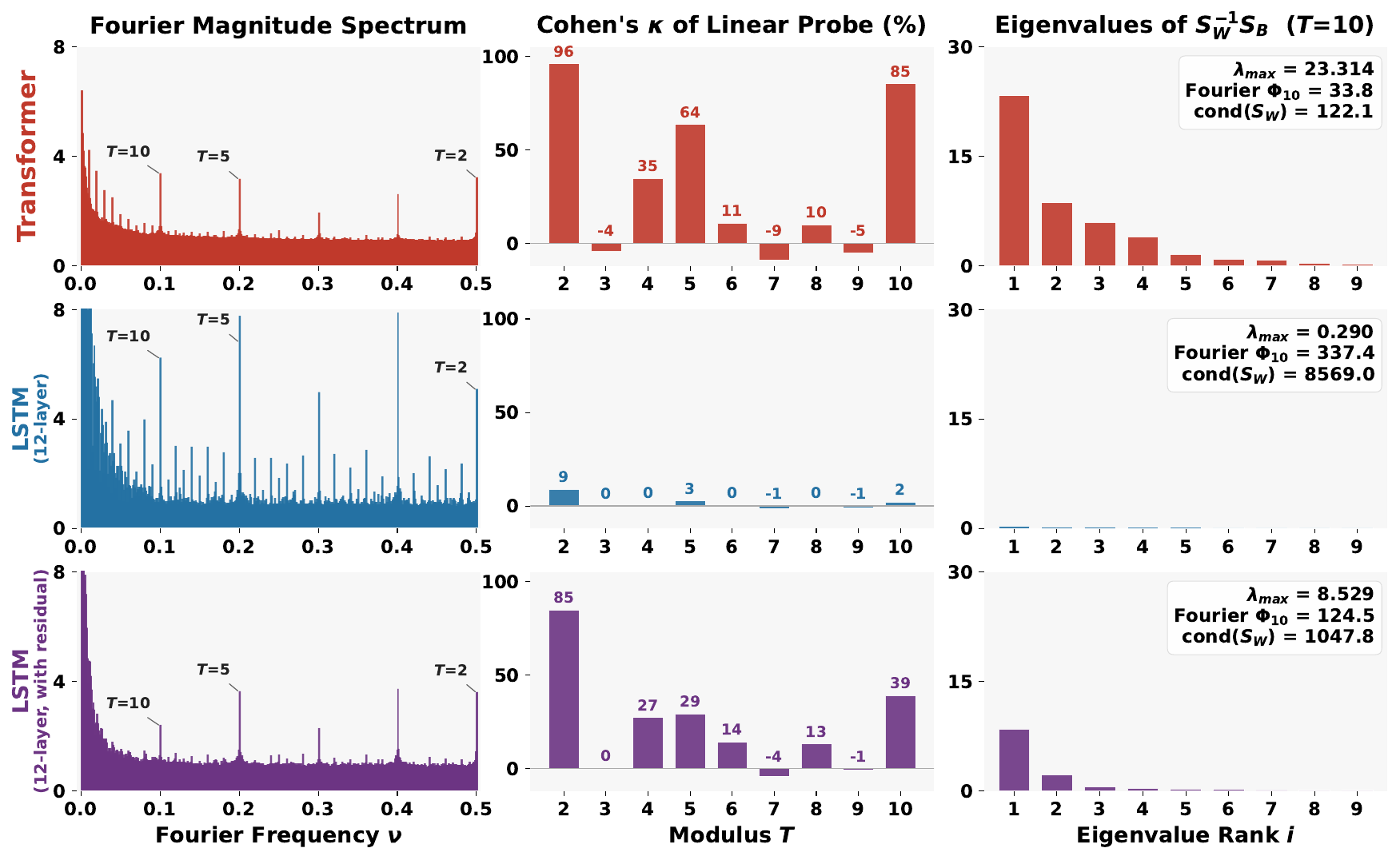}
    \caption{\textbf{Residual connections let the LSTM learn geometric structure.} Fourier spectrum (left), linear-probe Cohen's $\kappa$ (middle), and eigenvalues of $S_W^{-1}S_B$ at $T{=}10$ (right) for a Transformer, a 12-layer LSTM, and the same LSTM with residual connections. Residuals drop $\mathrm{cond}(S_W)$ from
    $8569$ to $1048$ and bring probing back from chance to $\kappa{=}85$ at $T{=}2$ and $39$ at $T{=}10$, still below the Transformer at $T{=}5,10$.}
    \label{fig:lstm_add_residual}
\end{figure}

\subsection{Data Perturbation Details} \label{app:data}
In this section, we describe in detail how we perturb data for each configurations in \Cref{tab:data_configs}.

\paragraph{Isolate-$k$ configuration.}
We design an \emph{isolate} configuration to test whether Fourier features and mod-$T$ probes can emerge when reducing the interactive between number tokens, even indirectly through intermediate text tokens across multiple layers. The key idea is to enforce a block-diagonal causal attention mask that partitions each sequence into segments, where each segment contains at most $k$ token. Concretely, given a tokenized sequence, we locate all positions containing number tokens and place segment boundaries at the midpoint of the text span between each consecutive pair of $k$ number tokens. This way, every number token still sees some surrounding context on both sides, but can never interact with any other number token, if they are not in the same segment. Within each segment, standard causal attention applies: position $i$ attends to position $j$ only if $j \leq i$ and both positions belong to the same segment. We also reset RoPE position IDs to zero at segment boundaries to avoid leaking positional information across segments, and mask the loss at boundaries so the model is not trained to predict across segment breaks. Importantly, we do not modify the training data at all. The tokenized sequences are identical to those used in the standard configuration; only the attention mask differs.

\paragraph{Context length $\ell$.}
To test the role of broad context in Fourier feature formation and geometric emergence, we train models whose effective context is limited to a fixed window of $\ell$ consecutive tokens. Each 1024-token sequence is reshaped into $\lfloor 1024 / \ell \rfloor$ independent subsequences of length $\ell$, each processed with its own standard causal attention mask. Tokens in one window cannot attend to tokens in any other window. This is equivalent to training on a corpus of short documents of length $\ell$. We experiment with $\ell \in \{2, 4, 8, 64\}$: a window of $\ell = 2$ reduces the model to learning bigram statistics (each token sees only the immediately preceding token), while $\ell = 4$ and $8$ permit short-range dependencies but still prevents any long-range co-occurrence patterns. $\ell = 64$ permits longer range dependencies but is still much shorter than \texttt{Original}'s context length of 1024. As with the isolate configuration, the underlying token sequence is unchanged; only the effective context window differs.

\paragraph{Swap numbers.}
This configuration dissociates number tokens from their original textual context while preserving natural number $n$-gram statistics. For each training sequence, we keep every text token in its exact position and replace the entire subsequence of number tokens with a contiguous, order-preserving slice of number tokens drawn from other documents in the corpus. Concretely, we pre-extract all number tokens from the full training set into a single in-memory stream, and for each sequence we substitute a randomly chosen contiguous segment from this pool. The replaced number tokens therefore retain realistic sequential patterns but lose their association with the surrounding text. This tests whether the text-number co-occurrence structure, rather than the number token statistics alone, drives Fourier feature and modular probe emergence.

\paragraph{Unigram replace.}
In the unigram configuration, every number token in the training data is independently replaced by a random draw from the corpus-wide marginal (unigram) distribution over number tokens. This destroys all sequential structure among numbers, such as their relationship to surrounding text and their co-occurrence with other numbers. But this type of perturbation exactly preserves the marginal frequency of each number token. If a number token $n$ appears with probability $p_n$ in the original corpus, it appears with the same probability in the perturbed data. By comparing to the original and swap-numbers configurations, the unigram ablation isolates whether co-occurrence statistics beyond simple token frequency are necessary for Fourier features and modular arithmetic to emerge.

\begin{figure}[t]
    \centering
    \includegraphics[width=\linewidth]{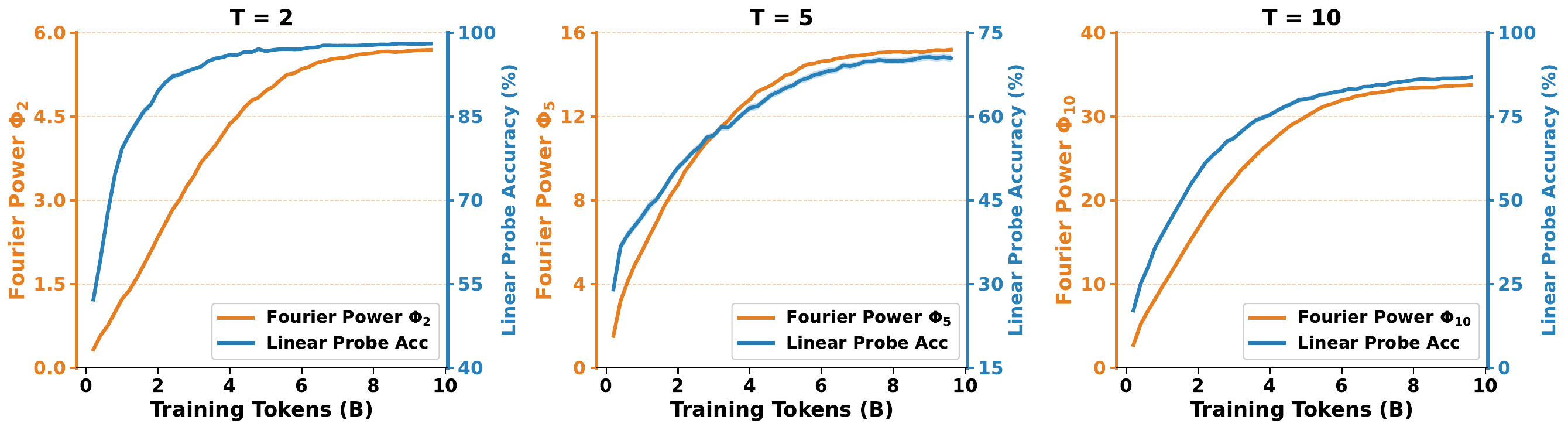}
    \caption{\textbf{Spectral and geometric convergence co-emerge gradually during pretraining.}
    Fourier power $\Phi_T$ (left axis) and linear probe accuracy (right axis) for a 300M Transformer trained with Muon, shown for $T = 2, 5, 10$.
    Both metrics increase smoothly throughout training with no phase transition, in contrast to the sudden emergence observed in grokking on modular arithmetic tasks in \citet{nanda2023progress}.}
    \label{fig:dynamics_transformer_original}
\end{figure}

\subsection{Training Dynamics} \label{app:appendix}
\Cref{fig:dynamics_transformer_original} shows that during language pretraining, both Fourier power $\Phi_T$ and linear probe accuracy increase smoothly from the start of training for $T = 2, 5, 10$, with no sudden phase transition.
This contrasts with grokking in modular arithmetic \citep{nanda2023progress}, where structured representations appear abruptly after prolonged memorization.
In language pretraining, the model continuously encounters diverse co-occurrence statistics rather than memorizing a fixed set of examples, so both tiers of convergence emerge gradually. In contrast, \Cref{fig:addition_dynamics} shows training dynamics for addition.

\begin{figure}[h]
    \centering
    \includegraphics[width=0.32\linewidth]{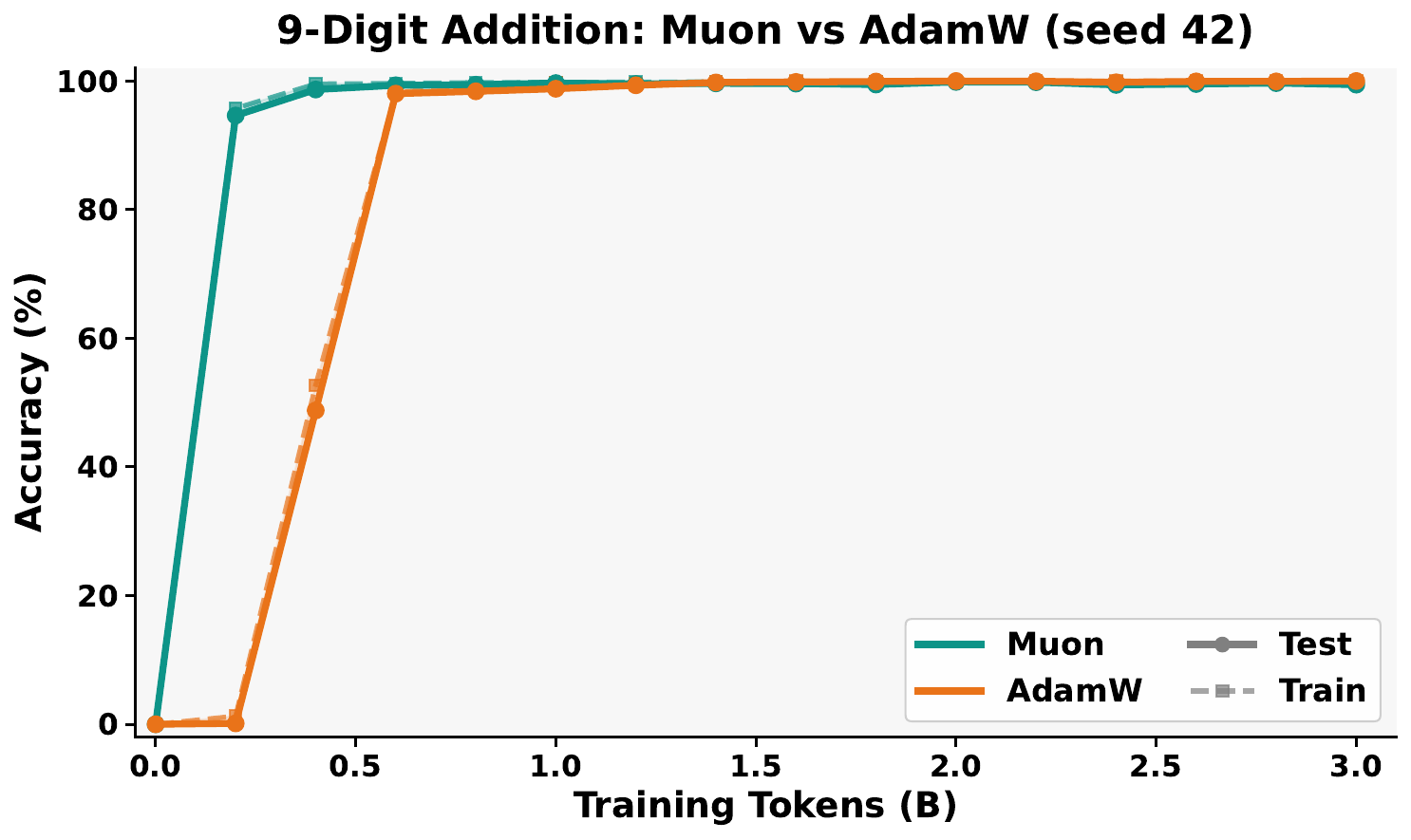}
    \includegraphics[width=0.32\linewidth]{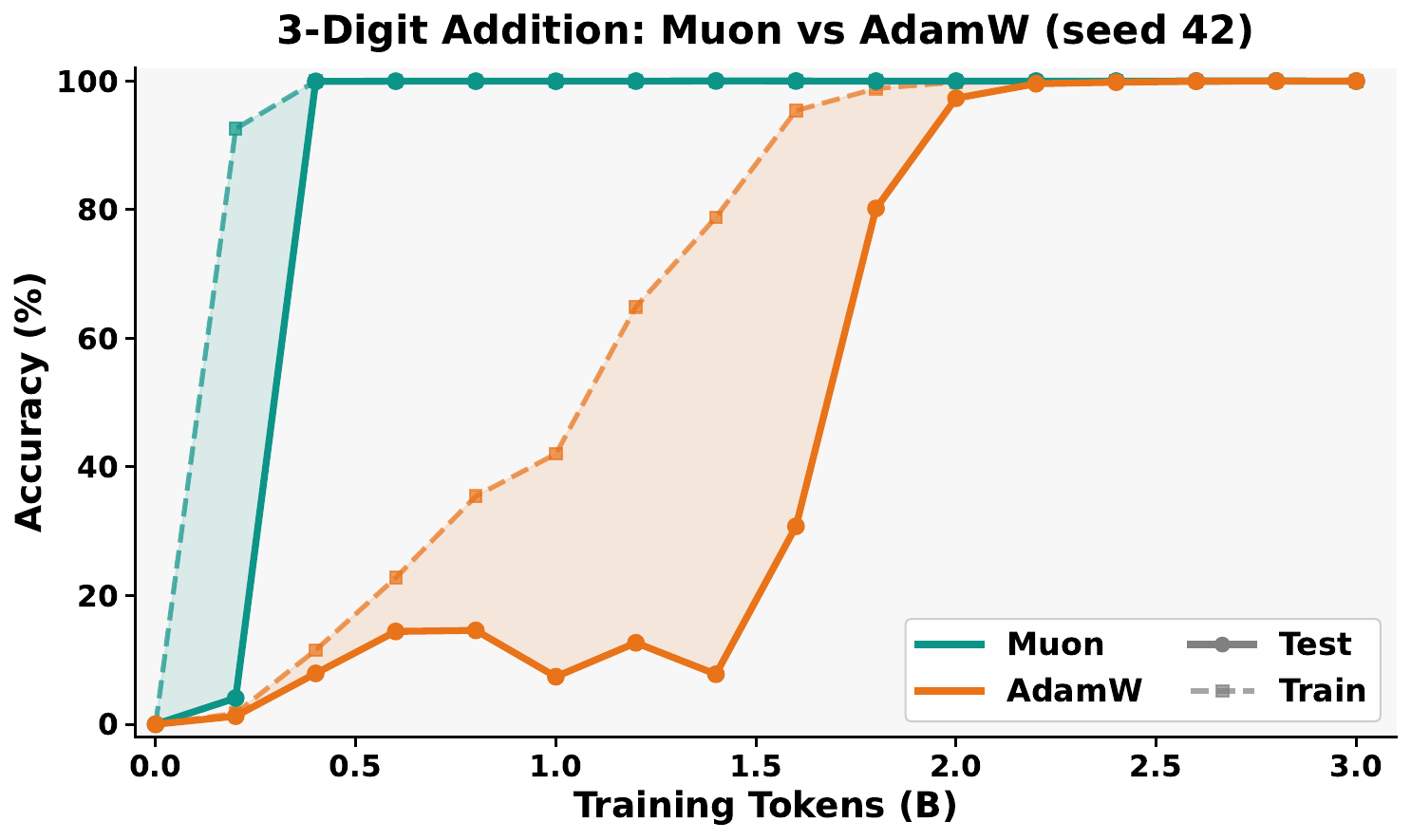}
    \includegraphics[width=0.32\linewidth]{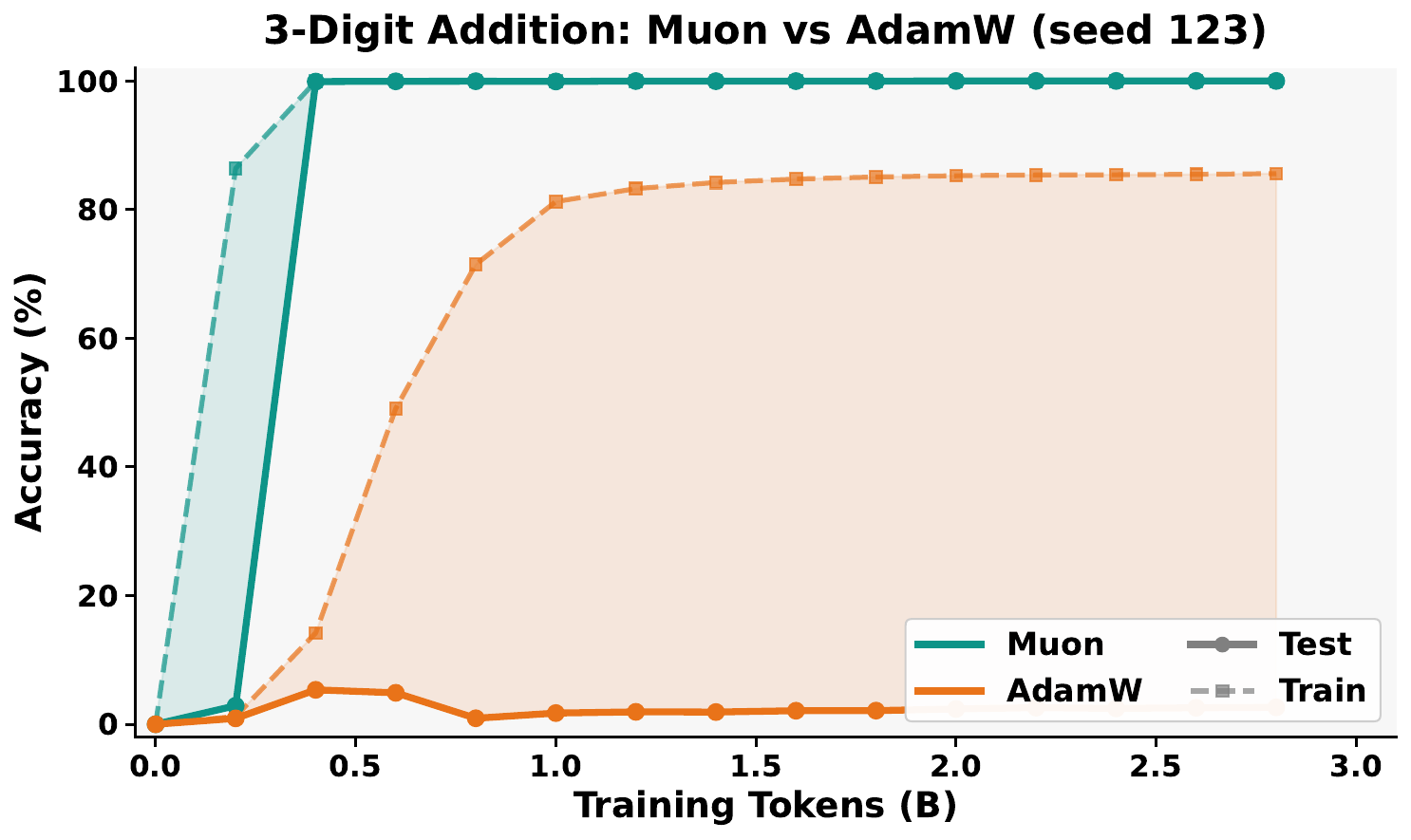}
    \caption{\textbf{Training dynamics for Transformers trained on addition (two seeds).} \textit{(Left)} In 9-digit addition, both Muon and AdamW converge smoothly  to near-perfect train and test accuracy, with no grokking phase. \textit{(Middle, Right)} In 3-digit addition, under seed 42, training accuracy reaches 100\% for both optimizers, but generalization is optimizer- and seed-dependent. AdamW exhibits a quick grokking under seed 42 (test accuracy jumps around 1.5--2B tokens) but not under seed 123 (where training accuracy can't reach 100\% and test accuracy remains random). This confirms that single-token addition imposes no consistent pressure toward structured representations.}
    \label{fig:addition_dynamics}
\end{figure}

\begin{figure}[t]
    \centering
    \includegraphics[width=0.9\linewidth]{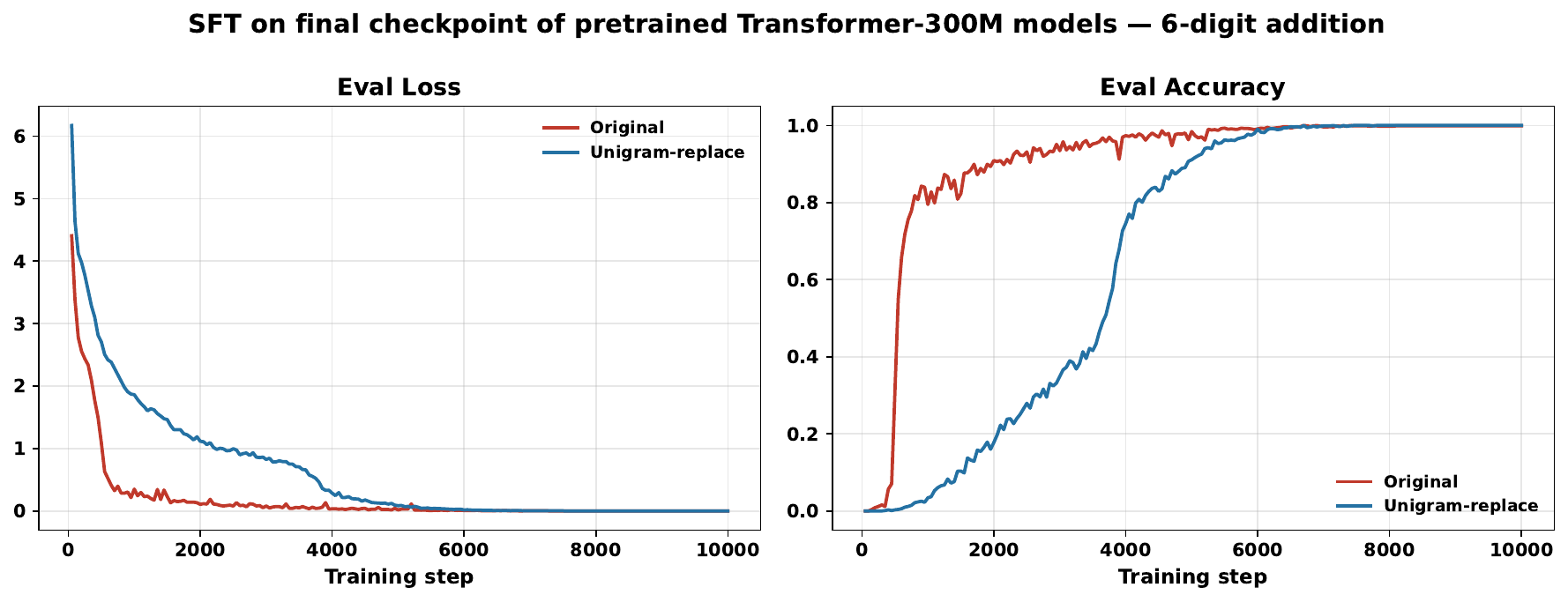}
    \includegraphics[width=0.9\linewidth]{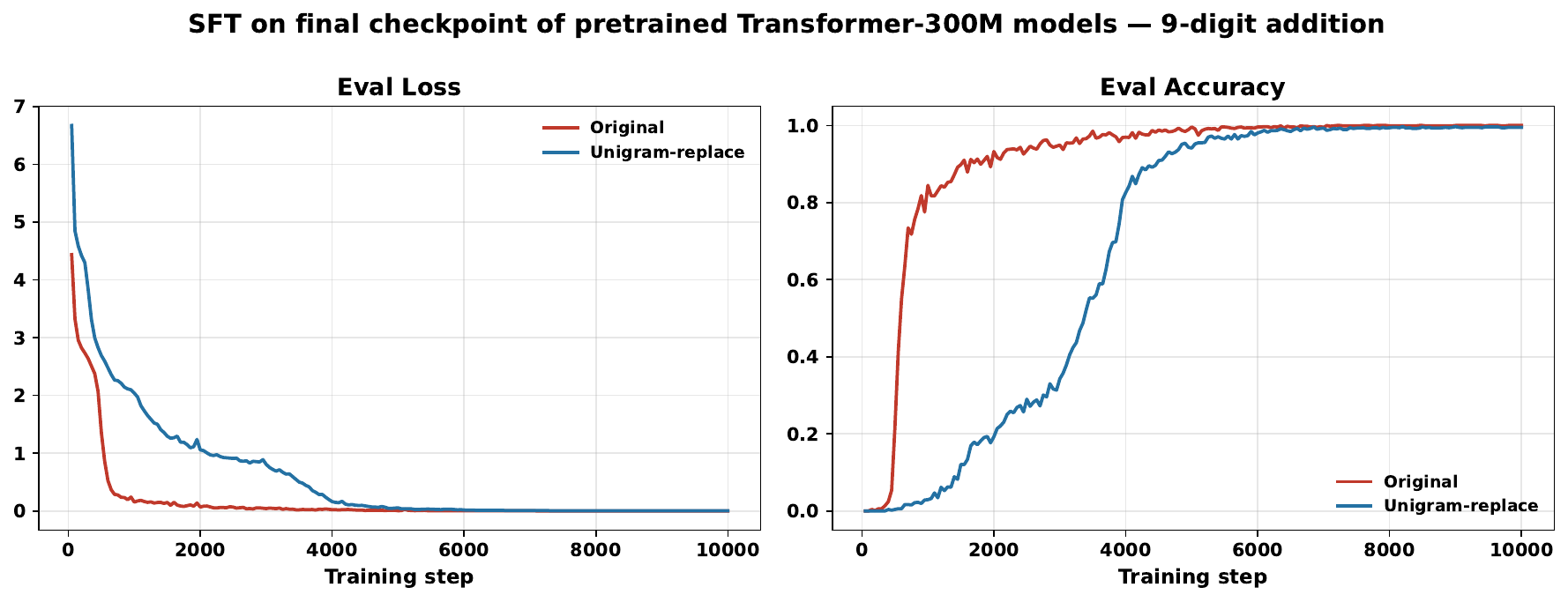}
    \caption{Geometric convergence tracks downstream arithmetic learning. SFT eval loss and accuracy on 6-digit (top) and 9-digit (bottom) addition for the Original and \texttt{Unigram-Replace} pretrained 300M Transformers. Only Original has geometric convergence, and it learns faster.}
    \label{fig:sft_addition_on_pretrained}
\end{figure}

\subsection{Relation to Downstream Tasks} \label{app:downstream}
Another thing worth noting is how spectral and geometric convergences are related to downstream tasks. Although it's hard to evaluate on real mathematical abilities of our 300M-parameter pretrained models since both the parameter count and tokens trained are too small to perform well on general mathematics, we instead design the following experiments. We take both an original and a perturbed pretrained models, where we choose \texttt{Unigram Replace} as the perturbed model, and perform supervised finetuning (SFT) on both 6-digit and 9-digit addition tasks, and track the training dynamics measured by accuracy and cross-entropy loss measured with 5,000 evaluation samples. As shown in \Cref{fig:sft_addition_on_pretrained}, no matter the task (6-digit or 9-digit addition), the original model (with better spectral and geometric convergence) converges much faster than the \texttt{Unigram Replace} variant. It demonstrates that better spectral and geometric convergence could imply better downstream performance.

\subsection{Non-Decimal Addition} \label{app:nondecimal}
Our paper centers around pretraining data with human texts where numbers are naturally represented in base-10, and it's interesting to understand if this implicit human bias leads to particular Fourier periods of 2, 5, and 10. To understand that, we design base-$10$, base-$15$ and base-$16$ versions of 6-digit addition tasks. Numbers are represented in these different bases and additions are up to 6-digit for each base. As shown in \Cref{fig:non_decimal}, the model trained on base-$10$ task could learn perfect probes on $\bmod 2, 5$ and $10$, and the model trained on base-$15$ could learn perfect probes on $\bmod 3, 5$ and $15$. We belived is because of the Chinese Remainder Theorem (CRT) where the base-$10$ class provide an inductive bias of $\mathbb Z / 10 \mathbb Z$ which is isomorphic to $\mathbb Z / 2 \mathbb Z \times \mathbb Z / 5 \mathbb Z$. Similarly, $\mathbb Z / 15 \mathbb Z$ under CRT is isomorphic to $\mathbb Z / 3 \mathbb Z \times \mathbb Z / 5 \mathbb Z$. However, as shown in \Cref{fig:non_decimal}, the model trained with base-$16$ learns no better than random guessing, and learns random fourier frequencies. This is because under CRT, one can't decompose $\mathbb Z / 16 \mathbb Z$ in to $\mathbb Z / p \mathbb Z \times \mathbb Z / q \mathbb Z$ where $p$ and $q$ are co-primes. Hence, the model can't have enough incentives to learn subgroups like $\mathbb Z / 2 \mathbb Z$. 

\begin{figure}[t]
    \centering
    \includegraphics[width=\linewidth]{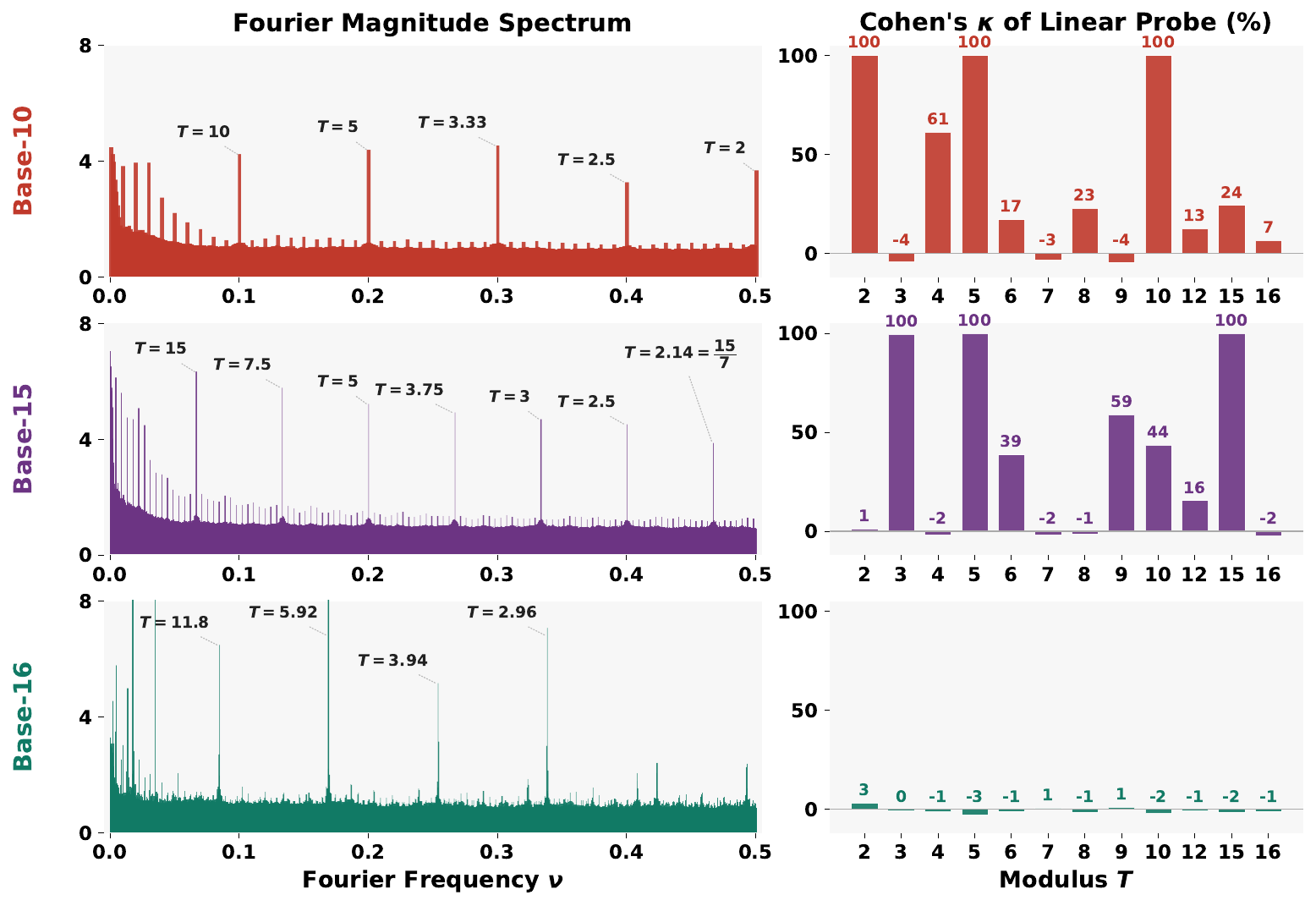}
    \caption{Base determines which moduli become separable, following the Chinese Remainder Theorem. Base-10 ($\cong\mathbb{Z}/2\times\mathbb{Z}/5$) learns mod-2,5,10. Base-15 ($\cong\mathbb{Z}/3\times\mathbb{Z}/5$) learns mod-3,5,15. Base-16 learns neither Fourier spikes nor probes.}
    \label{fig:non_decimal}
\end{figure}

\subsection{Further Experiments on Effects of Tokenizers} \label{app:tokenizer}
We want to further understand the effects of tokenizers, and we design the following experiment: we work with 6-digit addition, with two tokenizer, one is \texttt{3-digit} which is similar to the Llama-3 tokenizer that large numbers are chunked every 3 digits so that 6-digit numbers have 2 tokens, and the other is \texttt{6-digit} tokenizer where for 6-digit addition tasks, most numbers are tokenized into \textbf{single-token}s. As shown in \Cref{fig:tokenizer_ablate_6digit}, only the \texttt{3-digit} tokenizer learns the desired Fourier spectrum and $\bmod T$ probes for $T = 2, 5,$ and $10$, and the \texttt{6-digit} tokenizer could not learn Fourier spectrum or $\bmod T$ probes.

\begin{figure}[t]
    \centering
    \includegraphics[width=\linewidth]{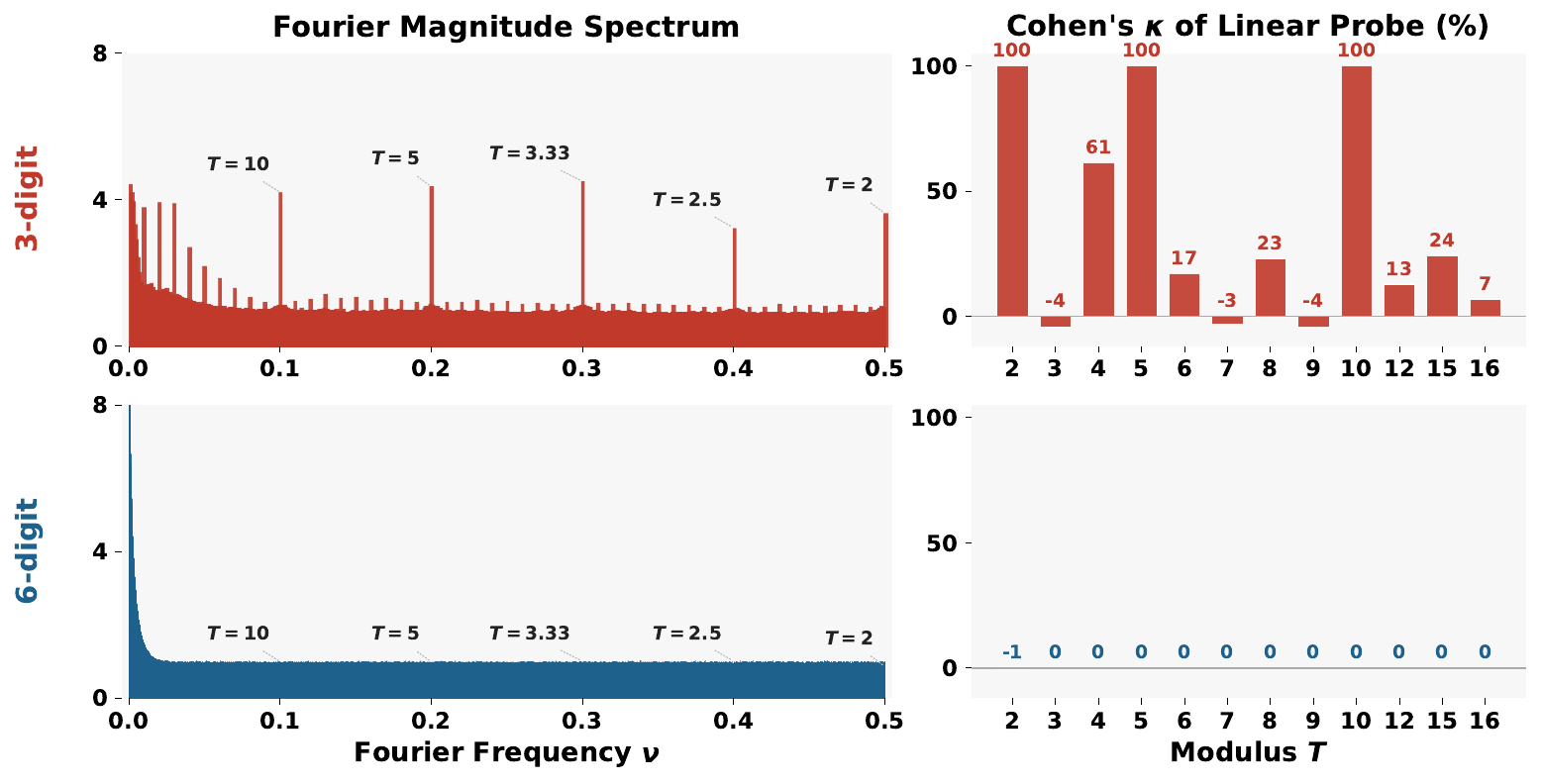}
    \caption{\textbf{Tokenizer decides convergence on 6-digit addition}. The 3-digit (Llama-3-style) tokenizer chunks numbers into multiple tokens and learns the mod-2,5,10 spectrum and probes. The 6-digit tokenizer makes numbers single tokens and learns nothing.}
    \label{fig:tokenizer_ablate_6digit}
\end{figure}

\subsection{Mod-$T$ Probing on Large Pretrained Models} \label{app:mod_probe_more}
To verify that geometric convergence holds beyond our 300M models, we run the same linear probe on $N$ single-token numbers (where $N$ is the total amount of single tokens numbers for each model) for eight pretrained models from 124M to 671B parameters (Table~\ref{tab:single-token-probe}). Mod-2, mod-5, and mod-10 are near-perfect across Transformers, SSMs, and an MoE. The moduli coprime to 10 behave differently. Mod-3, mod-7, and mod-9 sit at chance for the smaller models but become decodable
as scale grows or by newer models, reaching $\kappa{=}84.4$ (mod 7) and $92.0$ (mod 9) for DeepSeek-V3.

\begin{table}[t]
\centering
\caption{Mod-$T$ probe (Cohen's $\kappa$, \%) on large pretrained models.
Linear probe predicting
$n \bmod T$ from $e(n)$ for single-token numbers $0$--$999$.
Bold marks $\kappa \geq 99$.
Moduli coprime to $10$, such as 3, 7, and 9, stay mostly at chance.}
\label{tab:single-token-probe}
\small
\setlength{\tabcolsep}{4pt}
\begin{tabular}{llr >{\columncolor{modblue}}c cc >{\columncolor{modblue}}c cccc >{\columncolor{modblue}}c}
\toprule
Model & Params & $N$ & \multicolumn{9}{c}{$\kappa$ for $n \bmod T$ (\%)} \\
\cmidrule(lr){4-12}
 & & & 2 & 3 & 4 & 5 & 6 & 7 & 8 & 9 & 10 \\
\midrule
GPT-2           & 124M & 819  & \textbf{99.4}  & -3.5  & 64.5 & 98.5          & 9.6  & -9.4  & 36.9 & -6.9 & \textbf{99.9}  \\
GPT-2 XL        & 1.5B & 819  & \textbf{99.8}  & -10.4 & 77.6 & 98.4          & 4.0  & -12.8 & 32.0 & -8.7 & \textbf{99.4}  \\
Llama-3.2-1B    & 1.3B & 1000 & \textbf{100.0} & 34.3  & 95.7 & \textbf{99.7} & 35.2 & -7.1  & 29.9 & -2.9 & \textbf{99.8}  \\
Mamba-2.8B      & 2.8B & 877  & \textbf{99.8}  & \textbf{99.4} & 96.8 & \textbf{99.9} & \textbf{99.3} & 5.7 & 49.5 & 30.6 & \textbf{99.9} \\
Llama-3.2-3B    & 3.2B & 1000 & \textbf{100.0} & 73.7  & 98.8 & \textbf{99.9} & 60.6 & -10.1 & 53.6 & -1.1 & \textbf{99.9}  \\
Falcon-Mamba-7B & 7B   & 999  & \textbf{100.0} & 48.5  & 84.0 & \textbf{100.0}& 53.4 & -8.1  & 41.4 & 3.3  & \textbf{100.0} \\
Llama-3-8B      & 8B   & 1000 & \textbf{100.0} & \textbf{99.8} & 82.0 & \textbf{100.0} & 97.4 & -10.2 & 54.5 & 10.5 & \textbf{100.0} \\
DeepSeek-V3     & 671B & 1000 & \textbf{100.0} & 98.8  & 98.4 & \textbf{99.8} & 98.1 & 84.4  & 95.4 & 92.0 & \textbf{99.9}  \\
\bottomrule
\end{tabular}
\end{table}

\subsection{Multi-Token Number Representations} \label{app:multi_token_probe}
We extend the probe to numbers up to $999{,}999$, which the BPE tokenizer splits into 2--3 tokens, and read $n \bmod T$ off the last token's middle-layer hidden state (Table~\ref{tab:multi-token-probe}). This is a harder test, since the last token alone does not pin down $n$, so the model has to carry information across tokens through attention or recurrence.
Mod-2, 5, 10, and 100 stay near-perfect for the pretrained models, and separability extends to mod-1000 and mod-10000.
Mod-7 stays at chance, matching the single-token pattern for moduli coprime to 10.

\begin{table}[h]
\centering
\caption{Middle-layer probe $\kappa$ (\%) on multi-token numbers (up to
$999{,}999$, which has 2 to 3 tokens each).
We probe the last token's middle-layer hidden state on $n \bmod T$ tasks. These smaller-scale LLMs show strong probe performance on mod 2, 5, 10, 100, and 1000 but at chance performance at mod 7.}
\label{tab:multi-token-probe}
\small
\setlength{\tabcolsep}{5pt}
\begin{tabular}{lrrrrrrr}
\toprule
Model & $\bmod 2$ & $\bmod 5$ & $\bmod 7$ & $\bmod 10$ & $\bmod 100$ & $\bmod 1000$  \\
\midrule
GPT-2 (124M)     & 100.0 & 100.0 & -1.2 & 100.0 & 99.8  & 98.1   \\
GPT-2 XL (1.5B)  & 100.0 & 100.0 & -1.8 & 100.0 & 99.9  & 99.0   \\
Llama-3.2-1B     & 99.7  & 97.7  & -0.1 & 98.1  & 96.3  & 88.4 \\
Llama-3.2-3B     & 99.8  & 99.0  & -0.2 & 99.1  & 98.4  & 92.1 \\
Llama-3-8B       & 99.3  & 99.1  & -0.5 & 99.1  & 98.5  & 92.1 \\
Mamba-2.8B       & 100.0 & 100.0 & 2.7  & 100.0 & 100.0 & 98.7 \\
Falcon-Mamba-7B  & 100.0 & 100.0 & -0.4 & 100.0 & 100.0 & 92.7  \\
\bottomrule
\end{tabular}
\end{table}

\clearpage
\subsection{Modular Probe Results with MLP and RFM Probes} \label{app:mlp_rfm}
In this section, we present the modular probes similar to \Cref{fig:data,fig:arch_opt,fig:addition} but with RFM probes and 2-layer MLP probes with hidden layer of size 64. 
\begin{figure}[h]
    \centering
    \includegraphics[width=\linewidth]{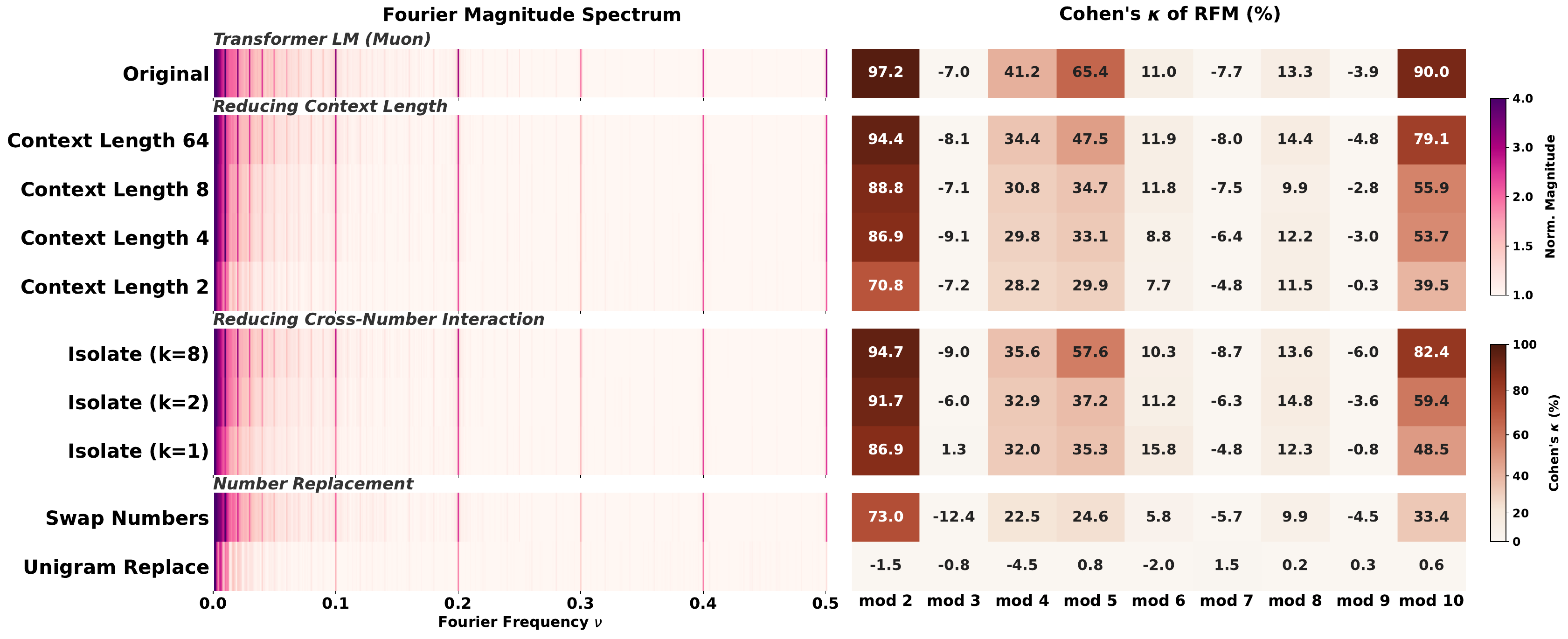}
    \includegraphics[width=\linewidth]{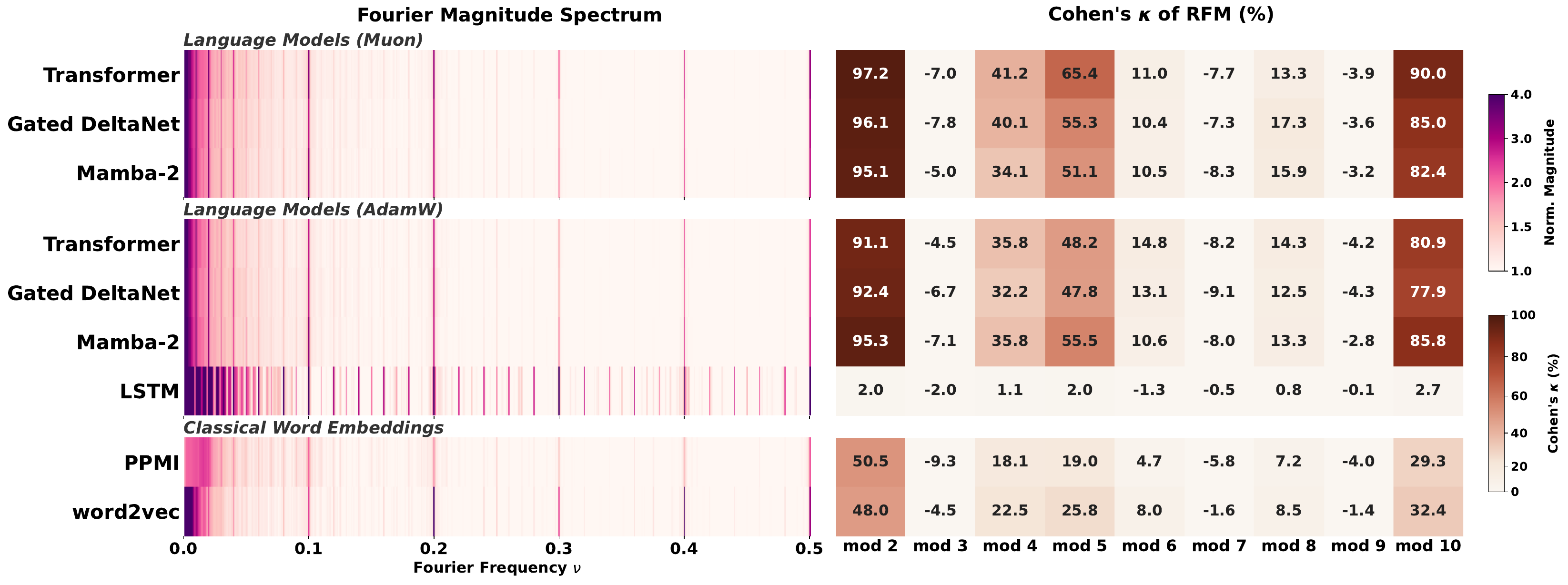}
    \includegraphics[width=\linewidth]{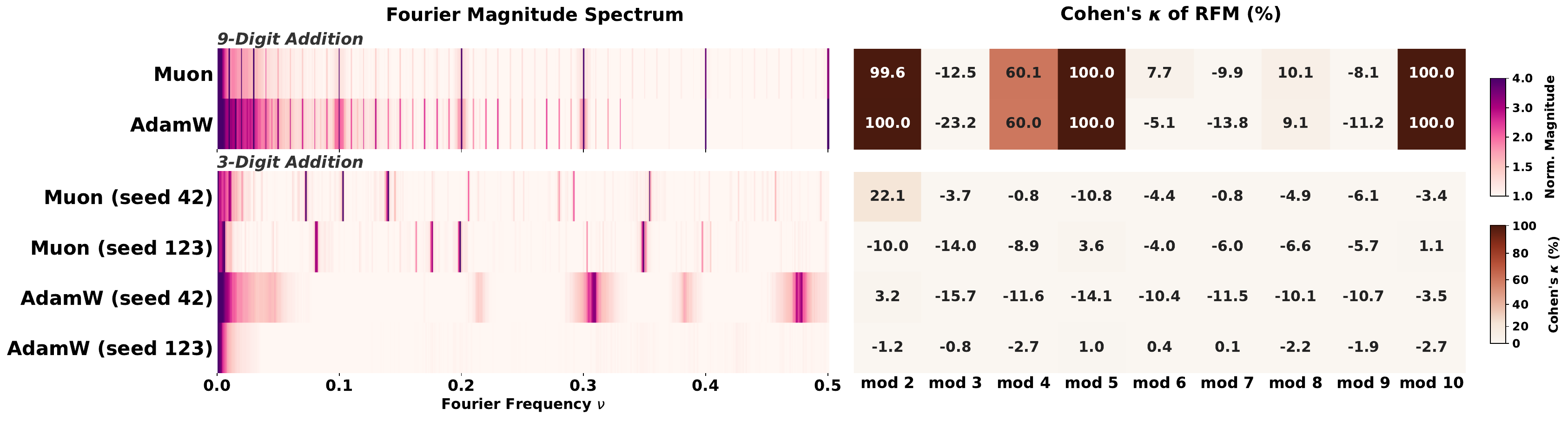}
    \caption{Structural Attribution Results with RFM probes.}
    \label{fig:probe_rfm}
\end{figure}

\begin{figure}[h]
    \centering
    \includegraphics[width=\linewidth]{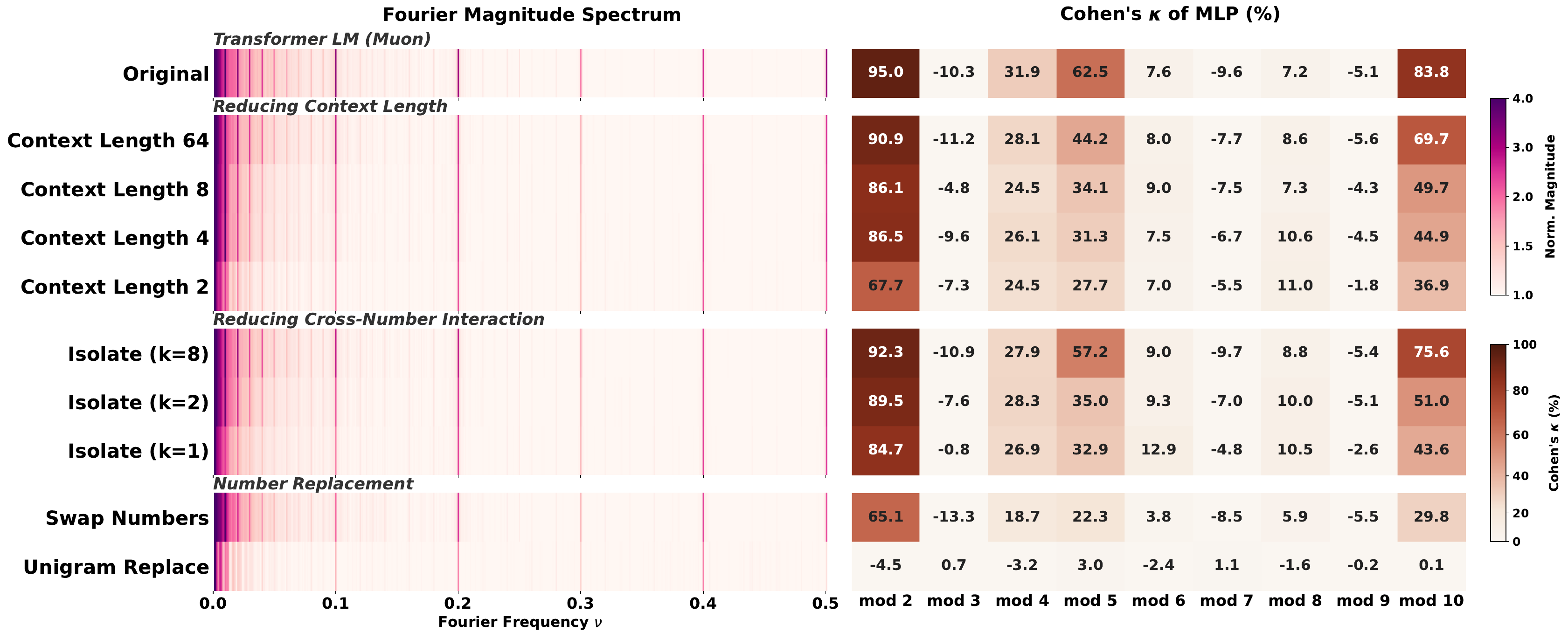}
    \includegraphics[width=\linewidth]{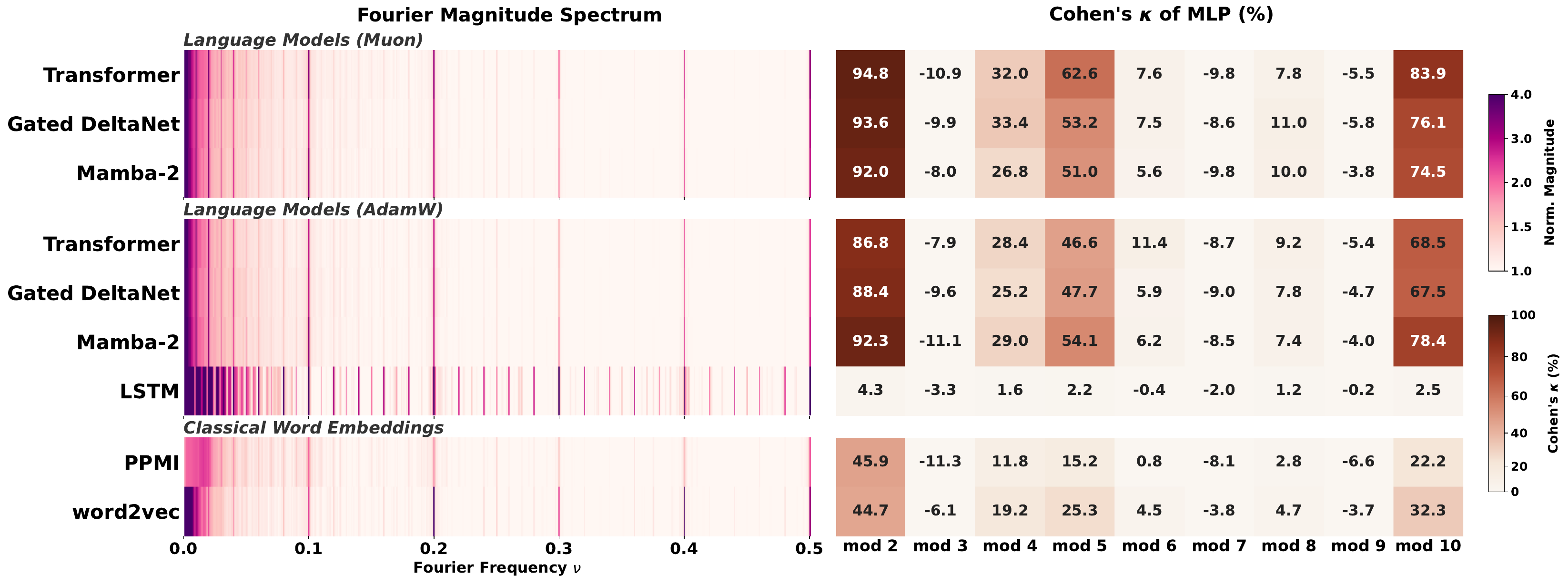}
    \includegraphics[width=\linewidth]{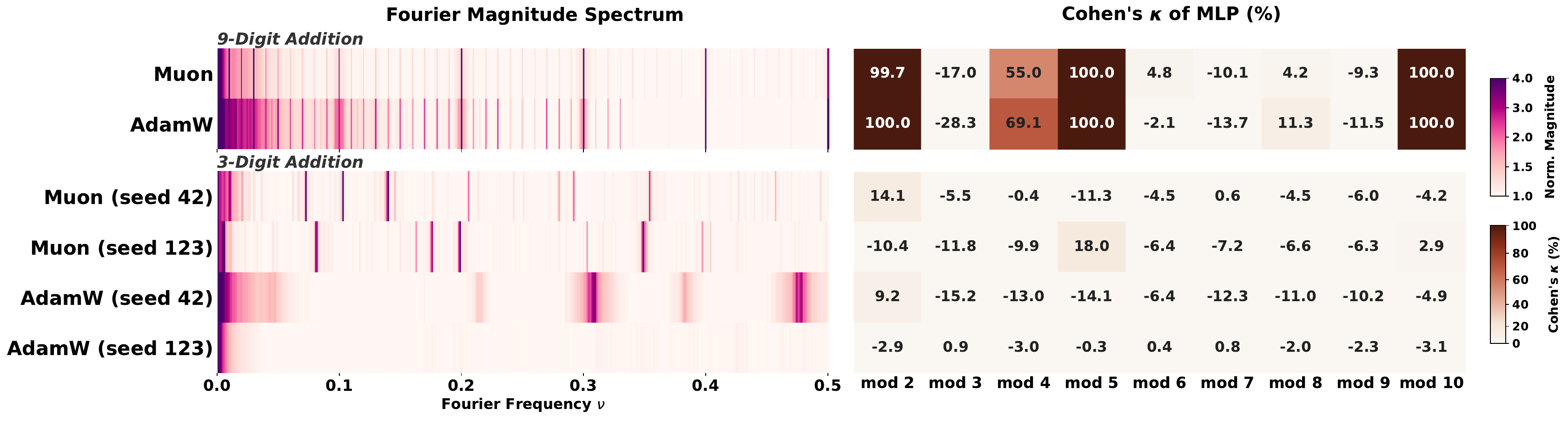}
    \caption{Structural Attribution Results with 2-layer MLP probes.}
    \label{fig:probe_mlp}
\end{figure}

\clearpage
\paragraph{Circular Probes for Transformers Trained on Addition.} To probe how number tokens are represented in the embedding layer, we train circular probes on the token embeddings for numbers.
Beyond $T$-class modular probes, which test linear separability in a $(T{-}1)$-dimensional subspace, a circular probe tests angular separability in 2-D: it learns a linear map $W \in \mathbb{R}^{d\times 2}$ projecting each embedding onto the unit circle, and classifies by cosine similarity to $m$ anchor directions at $\theta_k = 2\pi k/m$. 
\begin{figure}
    \centering
    \includegraphics[width=0.75\linewidth]{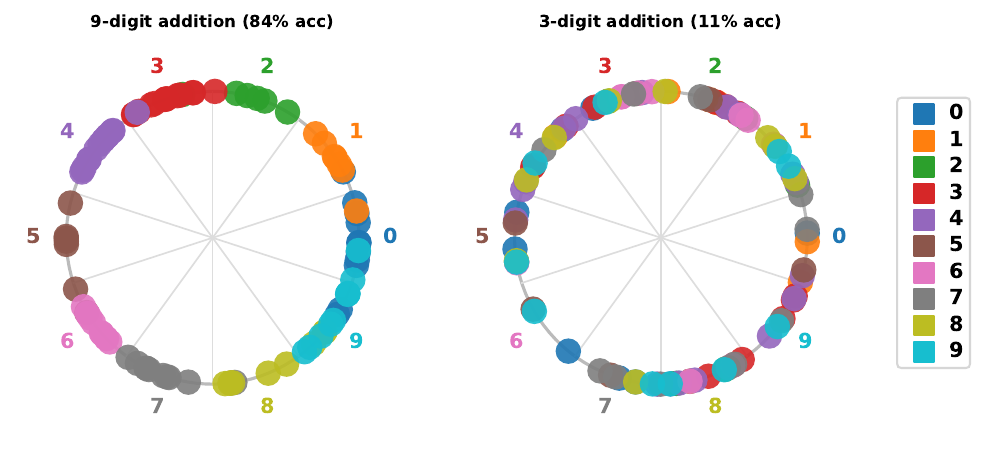}
    \caption{Circular probe projections of token embeddings onto the unit circle for mod-10
classification. Each point is a number token $n \in [0, 999]$ projected to 2-D by the probe
and normalized; color indicates $n \bmod 10$. (\textit{Left}) embeddings from a model trained
on 9-digit addition cluster sharply by residue class ($84\%$ test accuracy), indicating that
the token embedding layer has learned a geometrically organised, clock-like representation.
(\textit{Right}) embeddings from a model trained on 3-digit addition show no angular
structure ($11\%$ test accuracy, near chance), consistent with the absence of Fourier peaks
in Figure~\ref{fig:addition}.}
    \label{fig:circular_probe}
\end{figure}
\end{document}